\pdfoutput=1
\documentclass{article}

     \usepackage[preprint, nonatbib]{neurips_2020_modified}

\usepackage[utf8]{inputenc} 
\usepackage[T1]{fontenc}    
\usepackage{xcolor}
\usepackage{amsmath}
\usepackage{amsthm}
\usepackage{subcaption}
\usepackage{amssymb}
\definecolor{blue2}{rgb}{0.13,0.4,0.8} 
\definecolor{blue3}{rgb}{0,0,0.4} 
\definecolor{trueblack}{rgb}{0,0,0} 
\usepackage{subcaption}
\usepackage{tikz}
\usetikzlibrary{hobby}

\usepackage[colorlinks=true,
urlcolor=blue2,
linkcolor=trueblack,
citecolor=blue3]{hyperref}       
\usepackage{url}            
\usepackage{booktabs}       
\usepackage{amsfonts}       
\usepackage{nicefrac}       
\usepackage{microtype}      
\usepackage{float}

\usepackage{titlesec}  
\usepackage[title]{appendix} 

\usepackage{todonotes} 
\usepackage{thm-restate}
\usepackage{mathtools}
\usepackage{makecell}
\usepackage{subcaption}

\usepackage{stmaryrd}
\usepackage{apxproof}
\usepackage{enumitem}

\newtheorem{theorem}{Theorem}
\newtheorem{definition}[theorem]{Definition}

\newtheorem{proposition}[theorem]{Proposition}
\newtheorem*{proposition*}{}
\newtheorem{lemma}[theorem]{Lemma}
\newtheorem{claim}[theorem]{Claim}

\usepackage{amsmath,amsfonts,bm}

\def\eqref#1{equation~\ref{#1}}

\def\1{\bm{1}}

\def\vb{{\bm{b}}}

\def\vh{{\bm{h}}}

\def\vs{{\bm{s}}}

\def\vw{{\bm{w}}}
\def\vx{{\bm{x}}}
\def\vy{{\bm{y}}}
\def\vz{{\bm{z}}}

\def\mW{{\bm{W}}}

\DeclareMathAlphabet{\mathsfit}{\encodingdefault}{\sfdefault}{m}{sl}
\SetMathAlphabet{\mathsfit}{bold}{\encodingdefault}{\sfdefault}{bx}{n}

\newcommand{\gand}{\textsc{And}}
\newcommand{\gor}{\textsc{Or}}
\newcommand{\gnot}{\textsc{Not}}

\newcommand{\W}{\mathrm{W}}

\newcommand\false{\mathsf{false}}
\newcommand\true{\mathsf{true}}
\newcommand{\shp}{\text{$\#$}\text{\rm P}}
\newcommand{\ptime}{\text{\rm PTIME}}
\newcommand{\np}{\text{\rm NP}}
\newcommand{\conp}{\text{\rm co-NP}}

\newcommand{\M}{\mathcal{M}}
\newcommand{\K}{\mathcal{K}}
\newcommand{\C}{\mathcal{C}}
\newcommand{\Q}{\mathbb{Q}}
\newcommand{\relu}{\ensuremath{\operatorname{relu}}}
\newcommand{\step}{\ensuremath{\operatorname{step}}}
\newcommand{\NP}{\textsc{NP}}
\newcommand{\Ptime}{\textsc{P}}
\newcommand{\Maj}{{\sf Maj}} 

\newcommand{\dist}{\mathrm{d}}

\title{Model Interpretability through the Lens of Computational Complexity}

\author{
  Pablo Barceló$^{1,4}$, Mikaël Monet$^{2}$, Jorge Pérez$^{3,4}$, Bernardo Subercaseaux$^{3,4}$ \\ \\ 

  { $^1$ Institute for Mathematical and Computational Engineering, PUC-Chile }\\
  { $^2$ Inria Lille, France} \\
  { $^3$ Department of Computer Science, Universidad de Chile}\\
  { $^4$ Millennium Institute for Foundational Research on Data, Chile} \\
  {\footnotesize \texttt{pbarcelo@ing.puc.cl}, \texttt{mikael.monet@inria.fr}, \texttt{[jperez,bsuberca]@dcc.uchile.cl}}
}

\begin{document}

\maketitle
\begin{abstract}
In spite of several claims stating that some models are more interpretable than
others -- e.g.,  ``{linear models are more interpretable than deep neural
networks}'' -- we still lack a principled notion of interpretability to
formally compare among different classes of models.  We make a step towards
such a notion by studying whether folklore interpretability claims have a
correlate in terms of computational complexity theory.  We focus on
\emph{local post-hoc explainability queries} that, intuitively, attempt to
answer why individual inputs are classified in a certain way by a given model.
In a nutshell, we say that a class $\C_1$ of models is \emph{more
interpretable} than another class $\C_2$, if the computational complexity of
answering post-hoc queries for models in $\C_2$ is higher than for those in
$\C_1$.  We prove that this notion provides a good theoretical counterpart to
current beliefs on the interpretability of models; in particular, we show that
under our definition and assuming standard complexity-theoretical assumptions
(such as $\Ptime\neq \NP$), both linear and tree-based models are strictly more
interpretable than neural networks.  Our complexity analysis, however, does not
provide a clear-cut difference between linear  and tree-based models, as we
obtain different results depending on the particular {post-hoc explanations}
considered.  Finally, by applying a finer complexity analysis based on
parameterized complexity, we are able to prove a theoretical result suggesting
that shallow neural networks are more interpretable than deeper ones.

\end{abstract}

\section{Introduction}
\label{sec:introduction}
Assume a dystopian future in which the increasing number of submissions has
forced journal editors to use machine-learning systems for automatically accepting or rejecting papers.
Someone sends his/her work to the journal and the answer is a reject,
so the person demands an explanation for the decision.
The following are examples of three alternative ways in which the editor could provide an explanation for the rejection given by the system:
\begin{enumerate}
\item \emph{In order to accept the submitted paper it would be enough to include a better motivation and to delete at least two mathematical formulas.}
\item \emph{Regardless of the content and the other features of this paper, it was rejected because it has more than 10 pages and 
a font size of less than 11pt.}
\item \emph{We only accept 1 out of 20 papers that do not cite any other paper from our own journal. 
In order to increase your chances next time, please add more references.}
\end{enumerate}
These are examples of so called \emph{local post-hoc explanations}~\cite{BarredoArrieta2020,GuidottiMRTGP19,lipton2018mythos,molnar2019,Murdoch2019}. 
Here, the term ``local'' refers to explaining the verdict of the system for a particular input~\cite{GuidottiMRTGP19,Murdoch2019}, 
and the term ``post-hoc'' refers to interpreting the system after it has 
been trained~\cite{lipton2018mythos,molnar2019}.
Each one of the above explanations can be seen as a \emph{query} asked about a system and an input for it.
We call them \emph{explainability queries}.
The first query is related with the \emph{minimum change required} to obtain a desired outcome
(``what is the minimum change we must make to the article for it to be accepted by the system?'').
The second one is 
known as a \emph{sufficient reason}~\cite{shih2018symbolic}, and intuitively 
asks for a subset of the features of the given input that suffices to obtain the current verdict.
The third one, that we call \emph{counting completions}, relates to the probability of obtaining a particular 
output given the values in a subset of the features of the input.

In this paper we use explainability queries to formally compare the interpretability of machine-learning models.
We do this by relating the interpretability of a class of models (e.g.,~decision trees) 
to the \emph{computational complexity} of answering queries for models in that class.
Intuitively the lower the complexity of such queries is, the more interpretable the class is. 
We study whether this intuition 
provides an appropriate 
correlate to folklore wisdom on the interpretability of models~\cite{Gunning2019,lipton2018mythos,nguyen2020towards}.

\noindent{\bf Our contributions.}
We formalize the framework described above (Section~\ref{sec:framework}) and use it to 
perform
a theoretical study of the computational complexity
of three important types of explainability queries over three classes of models.
We focus on models often mentioned in the literature as extreme points in the interpretability spectrum: 
decision trees, linear models, and deep neural networks. 
In particular, we consider the class of 
\emph{free binary decision diagrams} (FBDDs), that generalize decision trees, the class
of \emph{perceptrons}, and the class of \emph{multilayer perceptrons} (MLPs) with ReLU 
activation functions. 
The instantiation of our framework for these classes is presented in Section~\ref{sec:instantiation}.

We show that, under standard complexity assumptions,
the computational problems associated to our interpretability queries 
are strictly less complex for~FBDDs than they are for~MLPs.
For instance, we show that for~FBDDs, the queries minimum-change-required and counting-completions 
can be solved in polynomial time, while for~MLPs these queries are, respectively, $\textsc{NP}$-complete and~\shp-complete
(where $\shp$ is the prototypical intractable complexity class for counting problems).
These results, together with results for other explainability queries, 
show that under our definition for comparing the interpretability of classes of models,~FBDDs are indeed more interpretable than~MLPs.
This correlates with 
the folklore
statement that tree-based models are more
interpretable than deep neural networks.
We prove similar results for perceptrons: most explainability queries that we consider are strictly less complex 
to answer for perceptrons than they are for~MLPs.
Since perceptrons are a realization of a {linear model}, our results give 
theoretical evidence for another folklore claim stating 
that linear models are more interpretable than deep neural networks.
On the other hand, the comparison between perceptrons and~FBDDs is not definitive and depends on the particular 
explainability query. We establish all our computational complexity results in Section~\ref{sec:complexity}.

Then, we observe that standard complexity classes are not enough to differentiate the interpretability 
of shallow and deep~MLPs.
To present a meaningful comparison, we then use the machinery of 
\emph{parameterized complexity}~\cite{downey2013fundamentals, FG06}, 
a theory that allows the classification of hard computational problems on a finer scale.
Using this theory, we are able to prove that there are explainability queries that are more difficult to solve for 
deeper~MLPs compared to shallow ones, thus giving theoretical evidence that shallow~MLPs are more interpretable.
This is the most technically involved result of the paper, that we think provides new insights on the complexity of interpreting deep neural networks. 
We present the necessary concepts and assumptions as well as a precise statement of this result in~Section~\ref{sec:p-complexity}.

Most definitions of interpretability in the literature are directly related to humans in a subjective manner  \cite{Biran2017ExplanationAJ, Doshi17, Miller2019}.
In this respect we do not claim that our complexity-based notion of interpretability is \emph{the} right notion of interpretability, and thus our results should be taken as a study of the correlation between a formal notion and the folklore wisdom regarding a subjective concept.
We discuss this and other limitations of our results in~Section~\ref{sec:discussion}.
We only present a few sketches for proofs in the body of the paper and 
refer the reader to the appendix for detailed proofs of all our claims.

\section{A framework to compare interpretability}
\label{sec:framework}
In this section we explain the key abstract components of our framework.
The idea is to introduce the necessary terminology to formalize our
notion of being \emph{more interpretable in terms of complexity}.

\paragraph{Models and instances.}
We consider an abstract definition of a model~$\mathcal{M}$ simply as a Boolean function
$\M : \{0,1\}^n \to \{0,1\}$. 
That is, we focus on binary classifiers with Boolean input features.
Restricting inputs and outputs to be Booleans makes our setting cleaner 
while still covering several relevant practical scenarios.
A class of models is just a way of grouping models together.
An \emph{instance} is a vector in~$\{0,1\}^n$ and represents a possible input for a model.
A \emph{partial instance} is a vector in~$\{0,1,\bot\}^n$, with~$\bot$ intuitively representing ``undefined'' components. 
A partial instance~$\vx \in \{0,1,\bot\}^n$ represents, in a compact way, the set of all instances in $\{0,1\}^n$ 
that can be obtained by replacing undefined components in $\vx$ with values in $\{0,1\}$. We call these the 
{\em completions} of $\vx$. 

\paragraph{Explainability queries.}
An \emph{explainability query} is a question that we ask about a model~$\M$ and a (possibly partial) instance~$\vx$, and 
refers to what the model $\M$ does on instance $\vx$. 
We assume all queries to be stated either as \emph{decision problems} (that is, \textsc{Yes}/\textsc{No} queries)
or as \emph{counting problems} (queries that ask, for example, how many completions of a partial instance
satisfy a given property).
Thus, for now we can think of queries simply as functions having models and instances as inputs.
We will formally define some specific queries in the next section, when we instantiate our framework.

\paragraph{Complexity classes.}
We assume some familiarity with the most common computational complexity classes of polynomial time~(\ptime) and 
nondeterministic polynomial time (\np), and with the notion of hardness and 
completeness for complexity classes under polynomial time reductions. 
In the paper we also consider the class~$\Sigma_2^p$,
consisting of those problems that can be solved in 
\np~if we further grant access to an oracle that solves \np~queries in constant time.  
It is strongly believed that $\ptime \subsetneq \np \subsetneq \Sigma_2^p$~\cite{arora2009computational},  
where for complexity classes~$\K_1$ and~$\K_2$ we have that $\K_1 \subsetneq \K_2$ means the following: 
problems in $\K_1$ can be solved in $\K_2$, 
but complete problems for~$\K_2$ cannot be solved in~$\K_1$.

While for studying the complexity of our decision problems the above classes suffice, for counting 
problems we will need another one. This will be the class 
$\shp$, which corresponds to problems that can be defined 
as counting the number of accepting paths of a 
polynomial-time nondeterministic Turing machine~\cite{arora2009computational}. 
Intuitively, $\shp$ is the counting class associated to $\np$: while the
prototypical $\np$-complete problem is checking if a propositional
formula is satisfiable (\textsc{SAT}), the prototypical $\shp$-complete problem
is counting how many truth assignments satisfy a propositional formula (\textsc{$\#$SAT}).
It is widely believed that $\shp$ is ``harder'' than $\Sigma_2^p$, which we write as $\Sigma_2^p \subsetneq \shp$.\footnote{One has to be careful with this notation, however, as 
$\Sigma_2^p$ and $\shp$ are complexity classes for problems of different sort: the former being for decision problems, and the latter for 
counting problems. Although this issue can be solved by considering the class $\textrm{PP}$, we skip these technical details as they are not fundamental for the paper and can be found in most complexity theory textbooks, such as that of Arora and Barak~\cite{arora2009computational}.}

\paragraph{Complexity-based interpretability of models.}
Given an explainability query~$Q$ and a class~$\C$ of models, we denote by~$Q(\C)$ the computational problem
defined by~$Q$ restricted to models in~$\C$.
We define next the most important notion for our framework: that of being \emph{more interpretable in terms of complexity}
(\emph{c-interpretable} for short). We will use this notion to compare among classes of models.

\begin{definition}
Let~$Q$ be an explainability query, and~$\C_1$ and~$\C_2$ be two classes of models.
We say that~$\C_1$ is \emph{strictly more c-interpretable than~$\C_2$ with respect to~$Q$}, if
the problem~$Q(\C_1)$ is in the complexity class~$\K_1$, the problem~$Q(\C_2)$ is hard for complexity class~$\K_2$, and 
$\K_1 \subsetneq \K_2$. 
\end{definition}

For instance, in the above definition one could take~$\K_1$ to be the \ptime\ class~and~$\K_2$ to be the $\np$ class, or $\K_1 = \np$ and $\K_2 = \Sigma_2^p$.


\section{Instantiating the framework and main results}
\label{sec:instantiation}
Here we instantiate our framework on three important classes of Boolean models 
and explainability queries, 
and then present our main theorems comparing such models in terms of 
c-interpretability.

\subsection{Specific models}
\label{subsec:models}

\paragraph{Binary decision diagrams.}
A \emph{binary decision diagram} (BDD~\cite{wegener2004bdds}) is a rooted
directed acyclic graph~$\mathcal{M}$ with labels on edges and nodes, verifying:
(i) each leaf is labeled with~$\true$ or with~$\false$; (ii) each internal node (a
node that is not a leaf) is labeled with an element of~$\{1,\ldots,n\}$; and
(iii) each internal node has an outgoing edge labeled~$1$ and another
one labeled~$0$.  
Every instance $\vx=(x_1,\ldots,x_n)\in \{0,1\}^n$ defines a unique path $\pi_\vx$ from the root to a leaf in $\M$, 
which satisfies the following condition: for every non-leaf node~$u$ in~$\pi_\vx$, if~$i$ is the label of~$u$, then 
the path~$\pi_\vx$ goes through the edge that is labeled with~$x_i$.
The instance $\vx$ is positive, i.e.,~$\mathcal{M}(\vx) \coloneqq 1$, 
if the label of the leaf in the path $\pi_\vx$
is~$\true$, and  negative otherwise.  The \emph{size}~$|\mathcal{M}|$
of~$\mathcal{M}$ is its number of edges.  A binary decision
diagram~$\mathcal{M}$ is \emph{free} (FBDD) if for every path from the root to a
leaf, no two nodes on that path have the same label. A \emph{decision tree} is
simply an FBDD whose underlying graph is a tree.

\paragraph{Multilayer perceptron (MLP).} 
A multilayer perceptron $\M$ with $k$ layers is defined by a
sequence of {\em weight} matrices $\mW^{(1)},\ldots,\mW^{(k)}$, {\em bias} vectors $\vb^{(1)},\ldots,\vb^{(k)}$, 
and {\em activation} functions~$f^{(1)},\ldots,f^{(k)}$. 
Given an instance $\vx$, we inductively define 
\begin{equation}\label{eq:mlp}
\vh^{(i)} \coloneqq f^{(i)}(\vh^{(i-1)}\mW^{(i)} + \vb^{(i)}) \quad \quad \quad (i \in \{1,\dots,k\}),
\end{equation}
assuming that $\vh^{(0)}\coloneqq \vx$. The output of $\M$ on $\vx$ is defined as 
$\M(\vx) := \vh^{(k)}$.
In this paper we assume all weights and biases to be rational numbers. 
That is, we assume that there exists a sequence of positive integers $d_0,d_1,\ldots,d_k$ 
such that $\mW^{(i)}\in \Q^{d_{i-1}\times d_i}$ and $\vb^{(i)}\in \Q^{d_i}$.
The integer $d_0$ is called the \emph{input size} of $\M$, and $d_k$ the \emph{output size}.
Given that we are interested in binary classifiers, we assume that $d_k=1$.
We say that an~MLP as defined above has $(k-1)$ \emph{hidden layers}.
The {\em size} of an~MLP $\M$, denoted by $|\M|$, is the total size of its weights and biases, in which the size of
a rational number $\nicefrac{p}{q}$ is $\log_2(p)+\log_2(q)$ (with the convention that $\log_2(0)=1$).

We focus on~MLPs in which all internal functions $f^{(1)},\dots,f^{(k-1)}$ are the ReLU function~$\relu(x)\coloneqq \max(0,x)$.
Usually, MLP binary classifiers are trained using the \emph{sigmoid} 
as the output function~$f^{(k)}$.
Nevertheless, when an MLP classifies an input (after training), it takes decisions by simply using 
the \emph{pre activations}, also called \emph{logits}.
Based on this and on the fact that we only consider already trained~MLPs, 
we can assume without loss of generality that the output function~$f^{(k)}$ is the \emph{binary step} function,
defined as~$\step(x)\coloneqq 0$ if~$x<0$, and~$\step(x)\coloneqq 1$ if~$x\geq 0$.

\paragraph{Perceptron.}
A perceptron is an MLP with no hidden layers (i.e., $k=1$). 
That is, a perceptron~$\mathcal{M}$ is defined by a pair~$(\mW,\vb)$
such that~$\mW\in\Q^{d\times 1}$ and~$\vb\in \Q$, and the output is~$\mathcal{M}(\vx)=\step(\vx\mW+\vb)$.
Because of its particular structure, a perceptron is usually defined as a pair~$(\vw,b)$ with~$\vw$ a rational vector and~$b$ a rational number.
The output of~$\mathcal{M}(\vx)$ is then~$1$ if and only if 
$\langle \vx,\vw\rangle + b\geq 0$, where~$\langle \vx,\vw\rangle$ denotes the dot product between~$\vx$ and~$\vw$.

\subsection{Specific queries}

Given instances~$\vx$ and~$\vy$, 
we define~$\dist(\vx,\vy) \coloneqq \sum_{i=1}^n|\vx_i-\vy_i|$ as the number of components in which~$\vx$ and~$\vy$ differ. 
We now formalize the minimum-change-required problem, which checks if the output of the model can be changed by flipping  
the value of at most~$k$ components in the input. 

\begin{center}
\fbox{\begin{tabular}{rl}
Problem: & \textsc{MinimumChangeRequired (MCR)} \\
Input: & Model $\mathcal{M}$, instance $\vx$, and~$k \in \mathbb{N}$ \\
Output: & \textsc{Yes}, if there exists an instance $\vy$ with $\dist(\vx,\vy) \leq k$ \\ 
& and $\M(\vx)\neq\M(\vy)$, and \textsc{No} otherwise
\end{tabular}}
\end{center}

Notice that, in the above definition, instead of ``finding'' the minimum change we state the problem as a \textsc{Yes}/\textsc{No}
query (a decision problem) by adding an additional input~$k \in \mathbb{N}$ and then asking for a change of size at most~$k$.
This is a standard way of stating a problem to analyze its complexity~\cite{arora2009computational}.
Moreover, in our results, when we are able to solve the problem in \ptime~then we can also output a minimum change, and it is clear that if the decision problem is hard then the optimization problem is also hard. Hence, we can indeed state our problems as decision problems without loss of generality.

To introduce our next query, recall that a partial instance is a vector~$\vy=(y_1,\ldots,y_n) \in \{0,1,\bot\}^n$, and a completion of 
it is an instance~$\vx=(x_1,\ldots,x_n) \in \{0,1\}^n$ such that for every~$i$ where~$y_i \in \{0,1\}$ it holds that~$x_i = y_i$.
That is, $\vx$ coincides with $\vy$ on all the components of~$\vy$ that are not $\bot$.
Given an instance~$\vx$ and a model~$\M$, a \emph{sufficient reason for~$\vx$ with respect to~$\M$}~\cite{shih2018symbolic} is a partial instance~$\vy$, such that 
$\vx$ is a completion of~$\vy$ and every possible completion~$\vx'$ of~$\vy$ satisfies~$\M(\vx')=\M(\vx)$.
That is, knowing the value of the components that are defined in~$\vy$ is enough to determine the output~$\M(\vx)$.
Observe that an instance~$\vx$ is always a sufficient reason for itself, and that~$\vx$ could have multiple (other) sufficient reasons.
However, given an instance~$\vx$, the sufficient reasons of~$\vx$ that are most interesting are those having the least possible number of defined components; indeed, it is clear that the less defined components a sufficient reason has, the more information it provides about the decision of~$\M$ on~$\vx$.
For a partial instance~$\vy$, let us write~$\|\vy\|$ for its number of components that are not~$\bot$.
The previous observations then motivate our next interpretability query.
\begin{center}
\fbox{\begin{tabular}{rl}
Problem: & \textsc{MinimumSufficientReason (MSR)} \\
Input: & Model $\mathcal{M}$, instance $\vx$, and~$k \in \mathbb{N}$ \\
Output: & \textsc{Yes}, if there exists a sufficient reason~$\vy$ for $\vx$ wrt.~$\M$ with $\|\vy\|\leq k$, \\
& and \textsc{No} otherwise

\end{tabular}}
\end{center}
As for the case of MCR, notice that we have formalized this interpretability query as a decision problem.
The last query that we will consider refers to counting the number of {positive completions} for a given partial instance.

\begin{center}
\fbox{\begin{tabular}{rl}
Problem: & \textsc{CountCompletions (CC)} \\
Input: & Model $\mathcal{M}$, partial instance $\vy$ \\
Output: & The number of completions $\vx$ of $\vy$ such that $\M(\vx)=1$
\end{tabular}}
\end{center}

Intuitively, this query informs us on the proportion of inputs that are accepted by the model, given that some particular features have been fixed; or, equivalently,
on the \emph{probability} that such an instance is accepted, assuming the other features to be uniformly and independently distributed.

\subsection{Main interpretability theorems}

We can now state our main theorems, which are illustrated in Figure~\ref{fig:main-results}. 
In all these theorems we use~$\C_\text{MLP}$ to denote the class of all models (functions from $\{0,1\}^n$ to $\{0,1\}$)
that are defined by MLPs, and similarly for~$\C_\text{FBDD}$ and~$\C_\text{Perceptron}$.
The proofs for all these results will follow as corollaries from the detailed complexity analysis that we present in Section~\ref{sec:complexity}.
We start by stating a strong 
separation between FBDDs and~MLPs, which holds for all the queries presented above.

\begin{theorem}\label{theo:fbdd-mlp}
$\C_\text{\emph{FBDD}}$ is strictly more c-interpretable than $\C_\text{\emph{MLP}}$
with respect to
\textsc{MCR}, \textsc{MSR}, and~\textsc{CC}.
\end{theorem}

For the comparison between perceptrons and MLPs, 
we can establish a strict separation for MCR and~MSR , but not for CC. In fact, 
CC has the same complexity for both classes of models, which means that none of these classes strictly 
``dominates'' the other in terms of c-interpretability for~CC. 

\begin{theorem}\label{theo:perceptron-mlp}
$\C_\text{\emph{Perceptron}}$ is strictly more c-interpretable than $\C_\text{\emph{MLP}}$ 
with respect to
~\textsc{MCR} and \textsc{MSR}. In turn, the problems $\textsc{CC}(\C_\text{\emph{Perceptron}})$ and 
$\textsc{CC}(\C_\text{\emph{MLP}})$ are both complete for the same complexity class. 
\end{theorem}

The next result shows that, in terms of c-interpretability, the relationship between FBDDs and perceptrons is not clear, 
as each one of them is strictly more c-interpretable than the other for some explainability query.

\begin{theorem}\label{prop:perceptron-fbdd}
The problems~$\textsc{MCR}(\C_\text{\emph{FBDD}})$
and~$\textsc{MCR}(\C_\text{\emph{Perceptrons}})$ are both in~\ptime.
However,~$\C_\text{\emph{Perceptron}}$ is strictly more c-interpretable
than~$\C_\text{\emph{FBDD}}$ with respect to~\textsc{MSR},
while~$\C_\text{\emph{FBDD}}$ is strictly more c-interpretable
than~$\C_\text{\emph{Perceptron}}$ with respect to~\textsc{CC}.
\end{theorem}

We prove these results in the next section, where for each query~$Q$ and class of models~$\mathcal{C}$ we pinpoint the exact complexity of the problem~$Q(\mathcal{C})$.

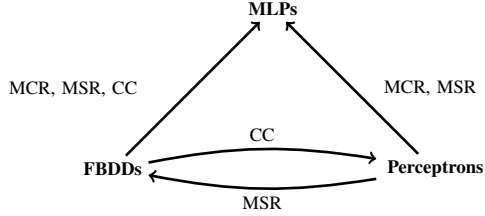
\begin{figure}[H]
	\centering
\scalebox{0.71}{
\begin{tikzpicture}
\begin{scope}[every node/.style={}]
	\node at (0, 0) (mlp) {{\bf MLPs}};
	\node at (-3, -3) (fbdd) {{\bf FBDDs}};
	\node at (3, -3) (p) {{ \bf Perceptrons}};
\end{scope}
\begin{scope}
	\draw[->, line width=.5mm] (p) edge[right] node[text width=3.0cm, right=1.3em] {MCR,  MSR} (mlp);
	\draw[->, line width=.5mm] (fbdd) edge[left] node[text width=3.3cm] {MCR, MSR, CC} (mlp); 
	\draw[->, line width=.5mm] (fbdd) edge[above, bend left=10] node {CC} (p);
	\draw[<-, line width=.5mm] (fbdd) edge[below, bend right=10] node {MSR} (p);  
  
\end{scope}
\end{tikzpicture}
}
	\caption{Illustration of the main interpretability results. Arrows depict that the pointed class of models is harder with respect to the query that labels the edge. We omit labels (or arrows) when a problem is complete for the same complexity class for two classes of models.}
	\label{fig:main-results}
\end{figure}

\section{The complexity of explainability queries}
\label{sec:complexity}
\begin{table}[h]
\begin{center}
\small
\begin{tabular}{lccc}
\toprule
                          & \textbf{FBDDs} & \textbf{Perceptrons} & \textbf{MLPs} \\
\midrule
\textsc{MinimumChangeRequired}   &   \ptime & \ptime     & $\textrm{NP}$-complete \\ 
\textsc{MinimumSufficientReason}  &  $\textrm{NP}$-complete &   \ptime   & $\Sigma_2^p$-complete \\ 
\textsc{CheckSufficientReason}  &   \ptime     &\ptime   & coNP-complete      \\ 
\textsc{CountCompletions}   & \ptime   & \shp-complete     & \shp-complete \\ 
\bottomrule
\end{tabular}
\end{center}
\caption{Summary of our complexity results.}
\label{table:results}
\end{table}

In this section we present our technical complexity results proving
Theorems~\ref{theo:fbdd-mlp},~\ref{theo:perceptron-mlp}, and \ref{prop:perceptron-fbdd}.
We divide our results in terms of the queries
that we consider.  We also present a few other complexity results that we find
interesting on their own.
A summary of the results is shown in Table~\ref{table:results}.
With the exception of Proposition~\ref{prp:ksuff}, items (1) and (3), the proofs for this section are relatively routine, were already known or follow from known techniques.
As mentioned in the introduction, 
we only present the main ideas of some of the proofs in the body of the paper, and a detailed exposition of each result can be found in the appendix.

\subsection{The complexity of \textsc{MinimumChangeRequired}}
\label{ssec:mcr}

In what follows we determine the complexity of the \textsc{MinimumChangeRequired} problem for the three classes of models that we consider.
\begin{restatable}{proposition}{kminchange}
\label{prp:k-minchange}
The \textsc{MinimumChangeRequired} query is (1) in $\ptime$ for FBDDs, (2) in $\ptime$ for perceptrons, and (3) $\np$-complete for MLPs.
\end{restatable}
\begin{proofsketch}
This query has been shown to be solvable in \ptime~for \emph{ordered} binary decision diagrams (OBDDs, a restricted form of FBDDs) by Shih et al.~\cite[Theorem 6]{shih2018formal} (the query is called \emph{robusteness} in the work of Shih et al.~\cite{shih2018formal}). 
We show that the same proof applies to FBDDs.
Recall that in an FBDD every internal node is labeled with a feature index in~$\{1,\ldots,n\}$.
The main idea is to compute a quantity~$\mathrm{mcr}_u(\vx) \in \mathbb{N}\cup \{\infty\}$ for every node~$u$ of the FBDD~$\M$.
This quantity 
represents the minimum number of features that we need to flip in~$\vx$ to 
modify the classification~$\M(\vx)$ 
if we are only allowed to change features associated with the paths from~$u$ to some leaf in the FBDD.
One can easily compute these values by processing the FBDD bottom-up.
Then the minimum change required for~$\vx$ is the value~$\mathrm{mcr}_r(\vx)$ where~$r$ is the root of~$\M$,
and thus we simply return \textsc{Yes} if~$\mathrm{mcr}_r(\vx)\leq k$, and~\textsc{No} otherwise.

For the case of a perceptron~$\M=(\vw,b)$ and of an instance~$\vx$, let us assume without loss of generality that~$\M(\vx)=1$.
We first define the \emph{importance}~$s(i) \in \mathbb{Q}$ of every input feature at position~$i$ as follows: if~$x_i =1$ then~$s(i) \coloneqq w_i$, and if~$x_i=0$ then~$s(i) \coloneqq -w_i$. Consider now the set~$S$ that contains the top~$k$
most important input features for which~$s(i) > 0$. We can easily show that it is enough to check whether flipping every feature in~$S$ 
changes 
the classification of~$\vx$, in which case we return~\textsc{Yes}, and return~\textsc{No} otherwise.

Finally, \np~membership of MCR for MLPs is clear: guess a partial instance~$\vy$ with~$\dist(\vx, \vy) \leq k$ and check in polynomial time that~$\M(\vx) \neq \M(\vy)$. We prove hardness with a simple reduction from the~\textsc{VertexCover} problem for graphs, which is known to be \np-complete. 
\end{proofsketch}

Notice that this result immediately yields Theorems~\ref{theo:fbdd-mlp}, \ref{theo:perceptron-mlp}, and  \ref{prop:perceptron-fbdd} for the case of MCR. 

\subsection{The complexity of \textsc{MinimumSufficientReason}}

We now study the complexity of \textsc{MinimumSufficientReason}. 
The following result yields Theorems~\ref{theo:fbdd-mlp}, \ref{theo:perceptron-mlp}, and  \ref{prop:perceptron-fbdd} for the case of MSR. 

\begin{restatable}{proposition}{ksuff}
\label{prp:ksuff}
The \textsc{MinimumSufficientReason} query is (1) $\np$-complete for FBDDs (and hardness holds already for decision trees), (2) in~$\ptime$ for perceptrons, and (3) $\Sigma_2^p$-complete for MLPs.
\end{restatable}
\begin{proofsketch}
Membership of the problem in the respective classes is easy. We show \np-completeness of the problem for FBDDs by a nontrivial reduction from the 
\np-complete problem of determining whether a directed acyclic graph has a dominating set of size at most~$k$~\cite{King_2005}. 
For a perceptron~$\M=(\vw,b)$ and an instance~$\vx$, assume without loss of generality that~$\M(\vx)=1$. 
As in the proof of Proposition~\ref{prp:k-minchange}, we consider the importance of every component of $\vx$, and prove that it is enough to check whether the $k$ most important features of $\vx$ are a sufficient reason for it, in which case we return~\textsc{Yes}, and simply return~\textsc{No} otherwise.
Finally, the~$\Sigma_2^p$-completeness for~MLPs is obtained again using a technical reduction from the problem called \textsc{Shortest Implicant Core}, defined and shown to be~$\Sigma_2^p$-complete by Umans~\cite{Umans2001}.
\end{proofsketch}

To refine our analysis, we also consider the natural problem of \emph{checking} if a given partial instance is a sufficient reason for an instance.

\begin{center}
\fbox{\begin{tabular}{rl}
Problem: & \textsc{CheckSufficientReason (CSR)} \\
Input: & Model $\mathcal{M}$, instance~$\vx$ and a partial instance $\vy$ \\
Output: & \textsc{Yes}, if $\vy$ is a sufficient reason for~$\vx$ wrt. $\M$, and \textsc{No} otherwise
\end{tabular}}
\end{center}

We obtain the following (easy) result. 

\begin{restatable}{proposition}{checksuff}
\label{prp:checksuff}
The query \textsc{CheckSufficientReason} is (1) in~$\ptime$ for FBDDs, (2)~in~$\ptime$ for perceptrons, and (3) $\conp$-complete for MLPs.
\end{restatable}

We note that this result for FBDDs already appears in~\cite{darwiche2002knowledge} (under the name of \emph{implicant check}).
Interestingly, we observe that this new query maintains the comparisons in terms of c-interpretability, in the sense that~$\C_\text{FBDD}$ and~$\C_\text{Perceptron}$ are
strictly more c-interpretable than $\C_\text{MLP}$ with respect to~CSR.

\subsection{The complexity of \textsc{CountCompletions}}

What follows is our main complexity result regarding the query \textsc{CountCompletions},
which yields Theorems~\ref{theo:fbdd-mlp}, \ref{theo:perceptron-mlp}, and  \ref{prop:perceptron-fbdd} for the case of CC. 
\begin{restatable}{proposition}{counting}
\label{prp:counting}
The query \textsc{CountCompletions} is (1) in $\ptime$ for FBDDs, (2) \shp-complete for perceptrons, and (3) \shp-complete for MLPs.
\end{restatable}

\begin{proofsketch}
Claim (1) is a a well-known fact that is a direct consequence of the definition of FBDDs; indeed, we can easily compute by bottom-up induction of the FBDD a quantity representing for each node the number of positive completions of the sub-FBDD rooted at that node (e.g.,
see~\cite{darwiche2002knowledge,wegener2004bdds}). We prove (2) by showing a reduction from
the \shp-complete problem \textsc{$\#$Knapsack}, i.e., counting the number of solutions to
a~$0/1$ knapsack input.\footnote{Recall that such an input consists of natural numbers (given in binary)
$s_1,\ldots,s_n, k \in \mathbb{N}$, and a solution to it is a set $S \subseteq
\{1,\ldots,n\}$ with $\sum_{i \in S} s_i \leq k$.}
For the last
claim, we show that~MLPs with ReLU activations can simulate arbitrary Boolean formulas, 
which allows us to directly conclude~(3) since counting the number of
satisfying assignments of a Boolean formula is \shp-complete.
\end{proofsketch}

\paragraph{Comparing perceptrons and MLPs.}
Although the query \textsc{CountCompletions} is \shp-complete for perceptrons, we can still show that the complexity goes down to \ptime~if we assume the weights and biases to be integers given in unary; this is commonly called \emph{pseudo-polynomial time}.
\begin{restatable}{proposition}{pseudo}
\label{prp:pseudo}
The query \textsc{CountCompletions} can be solved in pseudo-polynomial time for perceptrons (assuming the weights and biases to be integers given in unary).
\end{restatable}
\begin{proofsketch}
This is proved by first reducing the problem to \textsc{$\#$Knapsack}, and then using a classical dynamic
programming algorithm to solve \textsc{$\#$Knapsack} in pseudo-polynomial time.
\end{proofsketch}

This result establishes a difference between perceptrons and MLPs in terms of CC, as this query remains \shp-complete for the latter
even if weights and biases are given as integers in unary. Another difference is established by the fact that \textsc{CountCompletions} for perceptrons can be 
efficiently approximated, while this is not the case for MLPs. To present this idea, we briefly recall the notion 
of 
{\em fully polynomial randomized
approximation scheme} (FPRAS~\cite{jerrum1986random}), which 
is heavily used to refine the analysis of the complexity of \shp-hard problems. 
Intuitively, an FPRAS is a polynomial time algorithm that computes with high probability a
$(1-\epsilon)$-multiplicative approximation of the exact solution, for
$\epsilon > 0$, in polynomial time in the size of the input and in the parameter
$1 / \epsilon$. 
We show:

\begin{restatable}{proposition}{approx}
\label{prp:approx}
The problem \textsc{CountCompletions} restricted to perceptrons
admits an FPRAS (and the use of randomness is not even needed in this case). This is not the case for MLPs, on the other hand, at least under standard 
complexity assumptions. 
\end{restatable}

\section{Parameterized results for MLPs in terms of number of layers}
\label{sec:p-complexity}
In Section~\ref{ssec:mcr} we proved that the query \textsc{MinimumChangeRequired} is $\np$-complete for MLPs.
Moreover, a careful inspection of the proof reveals that MCR is already $\np$-hard for MLPs with only a few layers.
This is not something specific to MCR: in fact, all lower bounds for the queries studied in the paper in terms of~MLPs
hold for a small, fixed number of layers.
Hence, we cannot differentiate the interpretability of shallow and deep MLPs 
with the complexity classes that we have 
used
so far.

In this section, we show 
how
to 
construct a gap between the (complexity-based) interpretability of 
shallow and deep MLPs by considering refined complexity classes in our $c$-interpretability framework.
In particular, we use \emph{parameterized complexity}~\cite{downey2013fundamentals,FG06}, a branch of complexity theory that studies the difficulty of a problem in terms of multiple input parameters.
To the best of our knowledge, the idea of using parameterized complexity theory to establish a gap in the complexity of interpreting shallow and deep networks is new.

We first introduce the main underlying idea of parameterized complexity in terms of two classical graph
problems: \textsc{VertexCover} and \textsc{Clique}.
In both problems the input is a pair~$(G,k)$ with~$G$ a graph and~$k$ an integer.
In \textsc{VertexCover} we verify if there exists a set of nodes of size at most~$k$ that includes at least 
one endpoint for every edge in~$G$. 
In \textsc{Clique} we check if there exists a set of nodes of size at most~$k$ 
such that all nodes in the set are adjacent to each other. Both problems are known to be~$\np$-complete. 
However, this analysis treats~$G$ and~$k$ at the same level, which might not be fair in some practical situations in which~$k$ is much smaller than the size of~$G$. Parameterized complexity then studies how 
the complexity of the problems behaves when the input is only~$G$, and~$k$ is regarded as a small {\em parameter}. 

It happens to be the case that 
\textsc{VertexCover} and \textsc{Clique}, while both~$\np$-complete, have a different status in
terms of parameterized complexity. 
Indeed, \textsc{VertexCover} can be solved in time~$O(2^k \cdot |G|)$, which is polynomial in the size of the input~$G$
-- with the exponent not depending on~$k$ -- 
 and, thus, it is called \emph{fixed-parameter tractable}~\cite{downey2013fundamentals}. 
In turn, it is widely believed that there is no algorithm for \textsc{Clique} with 
time complexity~$O(f(k) \cdot \operatorname{poly}(G))$ -- with~$f$ being any computable function, that depends only on~$k$ --
and thus it is  \emph{fixed-parameter intractable}~\cite{downey2013fundamentals}.
To study the notion of fixed-parameter intractability, 
researchers on parameterized complexity have introduced the~$\W[t]$ complexity classes (with~$t\geq 1$), which form the so called
\emph{$\W$-hierarchy}.
For instance \textsc{Clique} is~$\W[1]$-complete~\cite{downey2013fundamentals}.
A core assumption in parameterized complexity is that~$\W[t]\subsetneq \W[t+1]$, for every~$t\geq 1$.

In this paper we will use
a related hierarchy,  
called 
the~$\W(\Maj)$-hierarchy \cite{Fellows}. 
We defer the formal definitions of these two hierachies to the appendix.
We simply mention here that both classes,~$\W[t]$ and~$\W(\Maj)[t]$, 
are closely related to logical circuits of depth~$t$.
The circuits that define the~$\W$-hierarchy use gates \textsc{And}, \textsc{Or} and \textsc{Not}, 
while circuits for~$\W(\Maj)$ 
use only the \textsc{Majority} gate (which outputs a~$1$ if more than half of its inputs are~$1$). 
Our result below applies to a special class of MLPs that we call restricted-MLPs (rMLPs for short), 
where we assume that the number of digits of each weight and bias in the MLP   is at most logarithmic in the number of neurons in the MLP
(a detailed exposition of this restriction can be found in the appendix).
We can now formally state the main result of this section.

\begin{restatable}{proposition}{mlpt}
\label{prp:mlpt}
For every $t\geq 1$ the \textsc{MinimumChangeRequired} query over rMLPs with~$3t+3$ layers is~$\W(\Maj)[t]$-hard and is contained in~$\W(\Maj)[3t+7]$.
\end{restatable}

By assuming that the $\W(\Maj)$-hierarchy is strict, we can use Proposition~\ref{prp:mlpt} to provide separations for rMLPs with
different numbers of layers. 
For instance, instantiating the above result with $t=1$ we obtain that for rMLPs with $6$ layers, the MCR
problem is in $\W(\Maj)[3t+7]=\W(\Maj)[10]$.
Moreover, instantiating it with $t=11$ we obtain that for rMLPs with $36$ layers, the MCR problem is $\W(\Maj)[11]$-hard.
Thus, assuming that $\W(\Maj)[10]\subsetneq \W(\Maj)[11]$ we obtain that rMLPs with $6$ layers are strictly more c-interpretable 
than rMLPs with $36$ layers. 
We generalize this observation in the following result.

\begin{proposition}
Assume that the~$\W(\Maj)$-hierarchy is strict.
Then for every $t\geq 1$ we have that rMLPs with $3t+3$ layers are strictly more c-interpretable than 
rMLPs with $9t+27$ layers wrt.~MCR.
\label{prp:result-layers}
\end{proposition}

\section{Discussion and concluding remarks}
\label{sec:discussion}
\paragraph*{Related work.}
The need for model interpretability in machine learning has been heavily advocated during the last few years,
with works covering 
theoretical 
and practical issues~\cite{BarredoArrieta2020, GuidottiMRTGP19, lipton2018mythos,molnar2019,Murdoch2019}.
Nevertheless, a formal definition of interpretability has remained elusive~\cite{lipton2018mythos}.
In parallel, a related notion of interpretability has emerged from the field of \emph{knowledge compilation}~\cite{darwiche2002knowledge,shi2020tractable,shih2018formal,shih2018symbolic,shih2019verifying}.
The intuition here is to construct a simpler and more interpretable model from a complex one.
One can then study the simpler model to understand how the initial one makes predictions.
Motivated by this, Darwiche and Hirth~\cite{darwiche2020reasons}
use variations of the notion of sufficient reason to explore the interpretability of 
\emph{Ordered BDDs} (OBDDs).
The FBDDs that we consider in our work generalize OBDDs, and thus, our results for sufficient reasons over FBDDs can 
be seen as generalizations of the results in~\cite{darwiche2020reasons}.
We consider FBDDs instead of OBDDs as FBDDs 
subsume decision trees, while OBDDs do not.
We point out here that the notion of sufficient reason for a Boolean classifier is the same as the notion of \emph{implicant} for a Boolean function, and that \emph{minimal} sufficient reasons (with minimailty refering to subset-inclusion of the defined components) correspond to \emph{prime implicants}~\cite{darwiche2002knowledge}.
We did not incorporate a study of minimal sufficient reasons (also called \emph{PI-explanations}) to our work due to space constraints.
In a contemporaneous work~\cite{marques2020explaining}, Marques-Silva et al. study the task of enumerating the minimal sufficient reasons of naïve Bayes and linear classifiers.
The queries \textsc{CountCompletions} and \textsc{CheckSufficientReason} have already been studied for FBDDs in~\cite{darwiche2002knowledge} (\textsc{CheckSufficientReason} under the name of \emph{implicant check}).
The query \textsc{MinimumChangeRequired} is studied in~\cite{shih2018formal} for OBDDs, where it is called \emph{robustness}.
Finally, there are papers exploring queries beyond the ones presented here~\cite{shi2020tractable,shih2018formal,shih2018symbolic}, such as \emph{monotonicity}, \emph{unateness}, \emph{bias detection}, \emph{minimum cardinality explanations}, etc.

\paragraph*{Limitations.}
Our framework provides a formal way of studying interpretability for classes of models, but 
still can be improved in several respects.
One of them is the use of a more sophisticated complexity analysis 
that is not so much focused on the \emph{worst case} complexity study propose here,  
but on identifying relevant parameters that characterize more precisely how difficult 
it is to interpret a particular class of models in practice. 
Also, in this paper we have focused on studying the local interpretability of models 
(why did the model make a certain prediction
on a given input?), but one could also study their \emph{global interpretability}, that is, making 
sense of the general relationships that a model has learned from the training data~\cite{Murdoch2019}.
Our framework can easily be extended to the global setting by considering queries about models, independent of the input
it receives. 
In order to avoid the difficulties of defining a general notion of interpretability~\cite{lipton2018mythos}, we have used explainability queries and their complexity as a formal proxy. 
Nonetheless, we do not claim that our notion of complexity-based interpretability is \emph{the} definitive notion of interpretability. 
Indeed, most definitions of interpretability are directly related to humans in a subjective manner  \cite{Biran2017ExplanationAJ, Doshi17, Miller2019}. Our work is thus to be taken as a study of the correlation between a formal notion of interpretability and the folk wisdom regarding a subjective concept.
Finally, even though the notion of complexity-based interpretability gives a precise way to compare models, our results show that it is still dependent on the particular set of queries that one picks. 
To achieve a more robust formalization of interpretability, one would then need to propose a more general approach that prescinds of specific queries. This is a challenging problem for future research.

\section{Broader impact}
Although interpretability as a subject may have a broad practical impact, our results in this paper are mostly theoretic,
so we think that this work does not present any foreseeable societal consequences.

\begin{ack}
Barceló and Pérez are funded by Fondecyt grant 1200967.
\end{ack}

\newpage 
\bibliographystyle{abbrv}
\bibliography{main}

\begin{thebibliography}{10}

\bibitem{Allender1989}
E.~Allender.
\newblock
  \href{https://www.computer.org/csdl/proceedings-article/focs/1989/063538/12OmNxwWouwo}{A
  note on the power of threshold circuits}.
\newblock In {\em 30th Annual Symposium on Foundations of Computer Science}.
  {IEEE}, 1989.

\bibitem{arora2009computational}
S.~Arora and B.~Barak.
\newblock {\em
  \href{https://theory.cs.princeton.edu/complexity/book.pdf}{Computational
  complexity: a modern approach}}.
\newblock Cambridge University Press, 2009.

\bibitem{BarredoArrieta2020}
A.~B. Arrieta, N.~D{\'{\i}}az-Rodr{\'{\i}}guez, J.~D. Ser, A.~Bennetot,
  S.~Tabik, A.~Barbado, S.~Garcia, S.~Gil-Lopez, D.~Molina, R.~Benjamins,
  R.~Chatila, and F.~Herrera.
\newblock \href{https://arxiv.org/abs/1910.10045}{Explainable artificial
  intelligence ({XAI}): Concepts, taxonomies, opportunities and challenges
  toward responsible {AI}}.
\newblock {\em Information Fusion}, 58:82--115, 2020.

\bibitem{berbeglia2009counting}
G.~Berbeglia and G.~Hahn.
\newblock
  \href{https://www.sciencedirect.com/science/article/pii/S0166218X09000857}{Counting
  feasible solutions of the traveling salesman problem with pickups and
  deliveries is\# P-complete}.
\newblock {\em Discrete Applied Mathematics}, 157(11):2541--2547, 2009.

\bibitem{Biran2017ExplanationAJ}
O.~Biran and C.~V. Cotton.
\newblock
  \href{https://pdfs.semanticscholar.org/02e2/e79a77d8aabc1af1900ac80ceebac20abde4.pdf}{Explanation
  and Justification in Machine Learning : A Survey}.
\newblock 2017.

\bibitem{Buss2006}
J.~F. Buss and T.~Islam.
\newblock
  \href{https://www.sciencedirect.com/science/article/pii/S0304397505006262?via%3Dihub}{Simplifying
  the weft hierarchy}.
\newblock {\em Theoretical Computer Science}, 351(3):303--313, 2006.

\bibitem{CLRS}
T.~H. Cormen, C.~E. Leiserson, R.~L. Rivest, and C.~Stein.
\newblock {\em
  \href{https://mitpress.mit.edu/books/introduction-algorithms-third-edition}{Introduction
  to Algorithms, Third Edition}}.
\newblock The MIT Press, 3rd edition, 2009.

\bibitem{darwiche2020reasons}
A.~Darwiche and A.~Hirth.
\newblock \href{https://arxiv.org/abs/2002.09284}{On the reasons behind
  decisions}.
\newblock {\em arXiv preprint arXiv:2002.09284}, 2020.

\bibitem{darwiche2002knowledge}
A.~Darwiche and P.~Marquis.
\newblock
  \href{http://www.cril.univ-artois.fr/~marquis/darwiche-marquis-jair02.pdf}{A
  knowledge compilation map}.
\newblock {\em Journal of Artificial Intelligence Research}, 17:229--264, 2002.

\bibitem{Doshi17}
F.~Doshi{-}Velez and B.~Kim.
\newblock \href{http://arxiv.org/abs/1702.08608}{A Roadmap for a Rigorous
  Science of Interpretability}.
\newblock {\em CoRR}, abs/1702.08608, 2017.

\bibitem{Downey1995}
R.~G. Downey and M.~R. Fellows.
\newblock
  \href{http://www.mrfellows.net/papers/DF95_FPTandCompletenessI.pdf}{Fixed-Parameter
  Tractability and Completeness I: Basic Results}.
\newblock {\em {SIAM} Journal on Computing}, 24(4):873--921, Aug. 1995.

\bibitem{downey2013fundamentals}
R.~G. Downey and M.~R. Fellows.
\newblock {\em
  \href{http://citeseerx.ist.psu.edu/viewdoc/download?doi=10.1.1.456.2729&rep=rep1&type=pdf}{Fundamentals
  of parameterized complexity}}, volume~4.
\newblock Springer, 2013.

\bibitem{Downey1998}
R.~G. Downey, M.~R. Fellows, and K.~W. Regan.
\newblock
  \href{http://www.mrfellows.net/papers/DFR98_CircuitComplexity.pdf}{Parameterized
  circuit complexity and the W hierarchy}.
\newblock {\em Theoretical Computer Science}, 191(1-2):97--115, Jan. 1998.

\bibitem{Fellows}
M.~Fellows, D.~Hermelin, M.~M\"{u}ller, and F.~Rosamond.
\newblock \href{http://www.mrfellows.net/papers/C82-democratic.pdf}{A purely
  democratic characterization of W[1]}.
\newblock In {\em Parameterized and Exact Computation}, pages 103--114.
  Springer Berlin Heidelberg.

\bibitem{Fellows2007CombinatorialCA}
M.~R. Fellows, J.~Flum, D.~Hermelin, M.~M{\"u}ller, and F.~A. Rosamond.
\newblock
  \href{http://www.mrfellows.net/papers/J74-TOCS-comb-circs-revised.pdf}{Combinatorial
  circuits and the W-hierarchy}.
\newblock 2007.

\bibitem{FG06}
J.~Flum and M.~Grohe.
\newblock {\em \href{http://yaroslavvb.com/upload/flum.pdf}{Parameterized
  complexity theory}}.
\newblock Texts in Theoretical Computer Science. An {EATCS} Series. Springer,
  2006.

\bibitem{Goldmann1998}
M.~Goldmann and M.~Karpinski.
\newblock
  \href{https://epubs.siam.org/doi/abs/10.1137/S0097539794274519}{Simulating
  threshold circuits by majority circuits}.
\newblock {\em {SIAM} Journal on Computing}, 27(1):230--246, 1998.

\bibitem{gopalan2011fptas}
P.~Gopalan, A.~Klivans, R.~Meka, D.~{\v{S}}tefankovic, S.~Vempala, and
  E.~Vigoda.
\newblock
  \href{https://www.cs.rochester.edu/u/stefanko/Publications/FOCS11.pdf}{An
  FPTAS for \#knapsack and related counting problems}.
\newblock In {\em 2011 IEEE 52nd Annual Symposium on Foundations of Computer
  Science}, pages 817--826. IEEE, 2011.

\bibitem{GuidottiMRTGP19}
R.~Guidotti, A.~Monreale, S.~Ruggieri, F.~Turini, F.~Giannotti, and
  D.~Pedreschi.
\newblock \href{https://arxiv.org/abs/1802.01933}{A survey of methods for
  explaining black box models}.
\newblock {\em {ACM} Comput. Surv.}, 51(5).

\bibitem{Gunning2019}
D.~Gunning and D.~Aha.
\newblock
  \href{https://www.aaai.org/ojs/index.php/aimagazine/article/view/2850}{{DARPA}'s
  explainable artificial intelligence ({XAI}) program}.
\newblock {\em {AI} Magazine}, 40(2):44--58, 2019.

\bibitem{jerrum1986random}
M.~R. Jerrum, L.~G. Valiant, and V.~V. Vazirani.
\newblock
  \href{http://www2.stat.duke.edu/~scs/Courses/Stat376/Papers/ConvergeRates/RandomizedAlgs/JerrumValiantVaziraniTCS1986.pdf}{Random
  generation of combinatorial structures from a uniform distribution}.
\newblock {\em TCS}, 43:169--188, 1986.

\bibitem{King_2005}
J.~A. King.
\newblock {\em
  \href{https://www.cs.mcgill.ca/~jking/papers/guarding_thesis.pdf}{Approximation
  algorithms for guarding 1.5 dimensional terrains}}.
\newblock PhD thesis, 2005.

\bibitem{lipton2018mythos}
Z.~C. Lipton.
\newblock \href{https://dl.acm.org/doi/pdf/10.1145/3236386.3241340}{The mythos
  of model interpretability}.
\newblock {\em Queue}, 16(3):31--57, 2018.

\bibitem{marques2020explaining}
J.~Marques-Silva, T.~Gerspacher, M.~C. Cooper, A.~Ignatiev, and N.~Narodytska.
\newblock \href{https://arxiv.org/abs/2008.05803}{Explaining Naive Bayes and
  Other Linear Classifiers with Polynomial Time and Delay}.
\newblock {\em arXiv preprint arXiv:2008.05803}, 2020.

\bibitem{Miller2019}
T.~Miller.
\newblock \href{https://arxiv.org/abs/1706.07269}{Explanation in artificial
  intelligence: Insights from the social sciences}.
\newblock {\em Artificial Intelligence}, 267:1--38, Feb. 2019.

\bibitem{molnar2019}
C.~Molnar.
\newblock {\em Interpretable machine learning}.
\newblock 2019.
\newblock \url{https://christophm.github.io/interpretable-ml-book/}.

\bibitem{Murdoch2019}
W.~J. Murdoch, C.~Singh, K.~Kumbier, R.~Abbasi-Asl, and B.~Yu.
\newblock \href{https://arxiv.org/abs/1901.04592}{Definitions, methods, and
  applications in interpretable machine learning}.
\newblock {\em Proceedings of the National Academy of Sciences},
  116(44):22071--22080, 2019.

\bibitem{nguyen2020towards}
T.~D. Nguyen, K.~E. Kasmarik, and H.~A. Abbass.
\newblock \href{https://arxiv.org/abs/2003.04675}{Towards interpretable deep
  neural networks: An exact transformation to multi-class multivariate decision
  trees}.
\newblock {\em arXiv}, pages arXiv--2003, 2020.

\bibitem{rizzi2019faster}
R.~Rizzi and A.~I. Tomescu.
\newblock
  \href{https://www.sciencedirect.com/science/article/pii/S0890540119300276}{Faster
  FPTASes for counting and random generation of Knapsack solutions}.
\newblock {\em Information and Computation}, 267:135--144, 2019.

\bibitem{shi2020tractable}
W.~Shi, A.~Shih, A.~Darwiche, and A.~Choi.
\newblock \href{https://arxiv.org/abs/2004.02082}{On tractable representations
  of binary neural networks}.
\newblock {\em arXiv preprint arXiv:2004.02082}, 2020.

\bibitem{shih2018formal}
A.~Shih, A.~Choi, and A.~Darwiche.
\newblock \href{http://proceedings.mlr.press/v72/shih18a/shih18a.pdf}{Formal
  verification of Bayesian network classifiers}.
\newblock In {\em International Conference on Probabilistic Graphical Models},
  pages 427--438, 2018.

\bibitem{shih2018symbolic}
A.~Shih, A.~Choi, and A.~Darwiche.
\newblock \href{https://arxiv.org/abs/1805.03364}{A symbolic approach to
  explaining Bayesian network classifiers}.
\newblock {\em arXiv preprint arXiv:1805.03364}, 2018.

\bibitem{shih2019verifying}
A.~Shih, A.~Darwiche, and A.~Choi.
\newblock
  \href{https://www.google.com/url?sa=t&rct=j&q=&esrc=s&source=web&cd=3&ved=2ahUKEwjj7eujo_noAhX8SRUIHZJ-BewQFjACegQIBBAB&url=http%3A%2F%2Freasoning.cs.ucla.edu%2Ffetch.php%3Fid%3D193%26type%3Dpdf&usg=AOvVaw3PR_FY0kGzfMBfoGTbqSN8}{Verifying
  binarized neural networks by Angluin-style learning}.
\newblock In {\em International Conference on Theory and Applications of
  Satisfiability Testing}, pages 354--370. Springer, 2019.

\bibitem{Umans2001}
C.~Umans.
\newblock
  \href{https://pdfs.semanticscholar.org/e46c/b895f66ae8671bab35200b825c1fdbd1f740.pdf}{The
  minimum equivalent {DNF} problem and shortest implicants}.
\newblock {\em Journal of Computer and System Sciences}, 63(4):597--611, 2001.

\bibitem{wegener2004bdds}
I.~Wegener.
\newblock
  \href{https://www.sciencedirect.com/science/article/pii/S0166218X0300297X}{BDDs—design,
  analysis, complexity, and applications}.
\newblock {\em Discrete Applied Mathematics}, 138(1-2):229--251, 2004.

\end{thebibliography}


\clearpage

\titleformat{\section}{\large\bfseries}{\appendixname~\thesection .}{0.5em}{}
\begin{appendices}
\begin{center}
{\Huge Appendix}
\end{center}

\vspace{1.5cm}

The appendix contains the proofs for all the results presented in the main document. It is organized as follows:

\begin{description}[leftmargin=3.0cm, labelindent=1cm]
	\item[Appendix~\ref{sec:simulation}] shows how MLPs can simulate Boolean circuits, which will be used in order to prove several propositions.
	\item[Appendix~\ref{sec:proof-5}] contains a proof of Proposition~\ref{prp:k-minchange}. 
	\item[Appendix~\ref{sec:proof-6}] contains a proof of Proposition~\ref{prp:ksuff}.
	\item[Appendix~\ref{sec:proof-7}] contains a proof of Proposition~\ref{prp:checksuff}.
	\item[Appendix~\ref{sec:proof-8}] contains a proof of Proposition~\ref{prp:counting}.
	\item[Appendix~\ref{sec:proof-9}] contains a proof of Proposition~\ref{prp:pseudo}.
	\item[Appendix~\ref{sec:proof-10}] contains a proof of Proposition~\ref{prp:approx}.
	\item[Appendix~\ref{sec:p-background}] contains a more detailed description of the parameterized complexity framework.
	\item[Appendix~\ref{sec:proof-11}] contains a proof of Proposition~\ref{prp:mlpt}.
	\item[Appendix~\ref{sec:proof-12}] contains a proof of Proposition~\ref{prp:result-layers}.
\end{description}

\section{Simulating Boolean formulas/circuits with MLPs}
\label{sec:simulation}
In this section we show that multilayer perceptrons can efficiently simulate arbitrary Boolean formulas. We will often use this result throughout the appendix to prove the hardness of our explainability queries over~MLPs. In fact, and this will make the proof cleaner, we will show a slightly more general result: that~MLPs can simulate arbitrary \emph{Boolean circuits}.
Formally, we show:

\begin{lemma}
\label{lem:circuits-to-MLPs}
Given as input a Boolean circuit~$C$, we can build in polynomial time an MLP~$\M_C$ that is equivalent to~$C$ as a Boolean function.
\end{lemma}
\begin{proof}
We will proceed in three steps. The first step is to build from~$C$ another
equivalent circuit~$C'$ that uses only what we call \emph{relu gates}. A relu gate is a gate that, on input~$\vx=(x_1,\ldots,x_m)\in \mathbb{R}^m$,
outputs~$\relu(\langle \vw, \vx \rangle + b)$, for some
rationals~$\vw_1,\ldots,\vw_m,b$. 
Observe that these gates do not necessarily output~$0$ or~$1$, and so the circuit~$C'$ might not be Boolean. However, we will ensure in the construction that the output of every relu gate in~$C'$, when given Boolean inputs (i.e.,~$\vx \in \{0,1\}^m$), is Boolean.
This will imply that the circuit~$C'$ is Boolean as well.
To this end, it is enough to show how to
simulate each original type of internal gate ($\gnot$, $\gor$,
$\gand$) by relu gates. We do so as follows:
	\begin{itemize}
		\item $\gnot$-gate: simulated with a relu gate with only one weight of value $-1$ and a bias of $1$. Indeed, it is clear that for~$x \in \{0,1\}$, we have that
        $\relu(-x+1) = \begin{cases}1 & \text{ if } x=0\\
                                    0 & \text{ if } x=1\end{cases}$.
		\item  $\gand$-gate of in-degree $n$: simulated with a relu gate with~$n$ weights, each of value~$1$, and a bias of value~$-(n-1)$. Indeed, it is clear that for~$\vx \in \{0,1\}^n$, we have that
    $\relu(\sum_{i=1}^n x_i - (n-1)) = \begin{cases}1 & \text{ if } \bigwedge_{i=1}^n x_i =1\\
                                    0 & \text{ otherwise}\end{cases}$.

		\item  $\gor$-gate of in-degree $n$: we first rewrite the $\gor$-gate with $\gnot$- and $\gand$-gates using De Morgan's laws, and then we use the last two items.
	\end{itemize}

The second step is to build a circuit~$C''$, again using only relu
gates, that is equivalent to~$C'$ and that is what we call \emph{layerized}.
This means that there exists a \emph{leveling function}~$l:C'' \to \mathbb{N}$
that assigns to every gate of~$C'$ a \emph{level} such that (i) every variable
gate is assigned level~$0$, and (ii) for any wire~$g \rightarrow g'$ (meaning
that~$g$ is an input to~$g'$) in~$C''$ we have that~$l(g') = l(g)+1$. To this
end, let us call a relu gate that has a single input and weight~$1$ and
bias~$0$ an \emph{identity gate}, and observe then that the value of an
identity gate is the same as the value of its only input, when this input is in~$\{0,1\}$.  We will obtain~$C''$
from~$C'$ by inserting identity gates in between the gates of~$C'$, which
clearly does not change the Boolean function being computed. We can do so naïvely as
follows. First, we initialize~$l(g)$ to~$0$ for all the variable gates~$g$
of~$C'$. We then iterate the following process: select a gate~$g$ such
that~$l(g)$ is undefined and such that~$l(g')$ is defined for every input~$g'$
of~$g$. Let~$g'_1,\ldots,g'_m$ be the inputs of~$g$, and assume that $l(g'_1)
\leq \ldots \leq l(g'_m)$. For every~$1 \leq i \leq m$, we insert a line
of~$l(g'_m) - l(g'_i)$ identity gates in between~$g'_i$ and~$g$, and we
set~$l(g) \coloneqq l(g'_m)+1$, and we set the levels of the identity gates
that we have inserted appropriately. It is clear that this construction can be
done in polynomial time and that the resulting circuit~$C''$ is 
layerized.

Finally, the last step is to transform~$C''$ into an MLP~$\M_C$ using only relu
for the internal activation functions and the step function for the output layer (i.e., what we simply call “an MLP” in the paper), and that respects the structure given by our definition
in Section~\ref{subsec:models} (i.e., where all neurons of a given layer are
connected to all the neurons of the preceding layer). We first deal with having a step gate instead of a relu gate for the output. To achieve this, we create a fresh identity gate~$g_0$, we set the output of~$C''$ to be an input of~$g_0$, and we set~$g_0$
to be the new output gate of~$C''$ (this does not change the Boolean function computed). We then replace~$g_0$ by a \emph{step gate} (which, we recall, on input~$x\in \mathbb{R}$ outputs~$0$ if~$x<0$ and~$1$ otherwise) with a weight of~$2$ and bias of~$-1$, which again does not change the Boolean function computed; indeed, for~$x\in \{0,1\}$, we have that
$\step(2x-1) = \begin{cases}1 & \text{ if } x =1\\
                                    0 & \text{ if } x=0\end{cases}$.

The level of~$g_0$ is one plus the level of the previous output gate of~$C''$.
Therefore, to make~$C''$ become a valid MLP, it is enough to do the following:
for every gate~$g$ of level~$i$ and gate~$g'$ of level~$i+1$, if~$g$ and~$g'$
are not connected in~$C''$, we make~$g$ be an input of~$g'$ and we set the
corresponding weight to~$0$. This clearly does not change the function
computed, and the obtained circuit can directly be regarded as an equivalent~MLP~$\M_C$.
Since the whole construction can be
performed in polynomial time, this concludes the proof.
\end{proof}

%
%
\section{Proof of Proposition~\ref{prp:k-minchange}}
\label{sec:proof-5}
In this section we prove Proposition~\ref{prp:k-minchange}. 
We recall its statement for the reader's convenience:
\kminchange*

We prove each item separately.

\begin{lemma}
	The \textsc{MinimumChangeRequired} query can be solved in linear time for FBDDs.
\end{lemma}

\begin{proof}
	
Let~$(\M, \vx, k)$ be an instance of \textsc{MinimumChangeRequired}, where~$\M$ is an FBDD. For every node~$u$ in~$\M$ we define~$\M_u$ to be the FBDD obtained by restricting~$\M$ to the nodes that are (forward-)reachable from~$u$; in other words,~$\M_u$ is the sub-FBDD rooted at~$u$. 
Then, we define~$\mathrm{mcr}_u(\vx)$ to be the minimum change required on~$\vx$ to obtain a classification under~$\M_u$ that differs from~$\M(\vx)$. More formally,

\[
	\mathrm{mcr}_u(\vx) = 
	\min \{ k' \mid  \text{there exists an instance } \vy \text{ such that } d(\vx, \vy) = k' \text{ and }
	\M_u(\vy) \neq \M(\vx)  \},
\]

with the convention that~$ \min \varnothing = \infty$.
Observe that, ($\dagger$) for an instance~$\vy$ minimizing~$k'$ in this equality, since the FBDD~$\M_u$ does not depend on the features associated to any node~$u'$ from the root of~$\M$ to~$u$ excluded, we have that for any such node~$\vy_{u'} = \vx_{u'}$ holds (otherwise~$k'$ would not be minimized).\footnote{We slightly abuse notation and write~$x_u$ for the value of
the feature of~$\vx$ that is indexed by the label of~$u$.} 
Let~$r$ be the root of~$\M$. Then, by definition we have that~$(\M, \vx, k)$ is a positive instance of \textsc{MinimumChangeRequired} if and only~$\mathrm{mcr}_r(\vx) \leq k$. We now explain how we can compute all the values~$\mathrm{mcr}_u(\vx)$ for every node~$u$ of~$\M$ in linear time.

By definition, if~$u$ is a leaf labeled with~$\true$ we have that~$\M_u(\vy) = 1$ for every~$\vy$, and thus if~$\M(\vx) = 0$ we get~$\mathrm{mcr}_u(\vx) = 0$, while if~$\M(\vx) = 1$ we get that~$\mathrm{mcr}_u(\vx) = \infty$. 
Analogously, if~$u$ is a leaf labeled with~$\false$, then~$\mathrm{mcr}_u(\vx)$ is equal to~$0$ if~$\M(\vx) = 1$ and to~$\infty$ otherwise. 

For the recursive case, we consider a non-leaf node~$u$. Let~$u_1$ be the node going along the edge labeled with~$1$ from~$u$, and~$u_0$ analogously.  Using the notation~$[x_u = a]$ to mean~$1$ if the feature of~$\vx$ indexed by the label of node~$u$ has value~$a\in \{0,1\}$, and~$0$ otherwise, and the convention that~$\infty + 1= \infty$, we claim that:
\[
	\mathrm{mcr}_u(\vx) = \min \Big([x_u = 1] + \mathrm{mcr}_{u_0}(\vx), [x_u = 0] + \mathrm{mcr}_{u_1}(\vx) \Big)
\]

Indeed, consider by inductive hypothesis that~$\mathrm{mcr}_{u_0}(\vx)$ and~$\mathrm{mcr}_{u_1}(\vx)$ have been properly calculated, and let us show that this equality holds. We prove both inequalities in turn:
\begin{itemize}
    \item We show that $\mathrm{mcr}_u(\vx) \leq \min \Big([x_u = 1] +
\mathrm{mcr}_{u_0}(\vx), [x_u = 0] + \mathrm{mcr}_{u_1}(\vx) \Big)$. It is
enough to show that both $\mathrm{mcr}_u(\vx) \leq [x_u = 1] +
\mathrm{mcr}_{u_0}(\vx)$ and $\mathrm{mcr}_u(\vx) \leq [x_u = 0] +
\mathrm{mcr}_{u_1}(\vx)$ hold. We only show the first inequality, as the other
one is similar. If~$\mathrm{mcr}_{u_0}(\vx) = \infty$ then clearly the
inequality holds, hence let us assume that~$\mathrm{mcr}_{u_0}(\vx) = k' \in
\mathbb{N}$. This means that there is an instance~$\vy'$ such that~$d(\vx,
\vy') = k'$ and such that~$\M_{u_0}(\vy') \neq \M(\vx)$. Furthermore, by the observation ($\dagger$) we have that 
for any node~$u'$ from the root of~$\M$ to~$u$ (included), we have~$\vy_{u'}=\vx_{u'}$. Therefore, the
instance~$\vy$ that is equal to~$\vy'$ but has value~$\vy_u = 0$ differs
from~$\vx$ in exactly~$k'' = [x_u = 1] + k'$, which implies
that~$\mathrm{mcr}_u(\vx) \leq  [x_u = 1] + \mathrm{mcr}_{u_0}(\vx)$.  Hence,
the first inequality is proven.
    \item We show that $\mathrm{mcr}_u(\vx) \geq \min \Big([x_u = 1] +
\mathrm{mcr}_{u_0}(\vx), [x_u = 0] + \mathrm{mcr}_{u_1}(\vx) \Big)$.  First,
assume that both $\mathrm{mcr}_{u_0}(\vx)$ and $\mathrm{mcr}_{u_1}(\vx)$ are
equal to~$\infty$. This means that every path in both~$\M_{u_0}$
and~$\M_{u_1}$ leads to a leaf with the same classification as~$\M(\vx)$. Then,
as every path from~$u$ goes either through~$u_0$ or through~$u_1$, it must be
that every path from~$u$ leads to a leaf with the same classification
as~$\M(\vx)$, and thus~$\mathrm{mcr}_u(\vx) = \infty$, and so the inequality
holds.  Therefore, we can assume that one of $\mathrm{mcr}_{u_0}(\vx)$ or
$\mathrm{mcr}_{u_1}(\vx)$ is finite. Let us assume without loss of generality
that ($\star$) $\min \Big([x_u = 1] + \mathrm{mcr}_{u_0}(\vx), [x_u = 0] +
\mathrm{mcr}_{u_1}(\vx) \Big) = [x_u = 1] + \mathrm{mcr}_{u_0}(\vx) \in
\mathbb{N}$ (the other case being similar).  Let us now assume, by way of
contradiction, that the inequality does not hold, that is, we have that ($\dagger \dagger$)
$\mathrm{mcr}_u(\vx) < [x_u = 1] + \mathrm{mcr}_{u_0}(\vx)$, and
let~$\vy$ be an instance such that~$\M_u(\vy) \neq \M_u(\vx)$ and~$\dist(\vx,\vy) = \mathrm{mcr}_u(\vx)$.
Thanks to ($\star$), we can assume wlog that~$\vy_u = 0$.
But then we would have that~$\mathrm{mcr}_{u_0}(\vx) \leq \mathrm{mcr}_u(\vx) - [x_u = 1]$, which contradicts ($\dagger \dagger$).
Hence, the second inequality is proven.
\end{itemize}

It is clear that the recursive function~$\mathrm{mcr}$ can be computed bottom-up in linear time, thus concluding the proof.
\end{proof}

\begin{lemma}
	The \textsc{MinimumChangeRequired} query can be solved in linear time for perceptrons.
	\label{lemma:mcr-perceptron}
\end{lemma}

\begin{proof}
Let~$(\M = (\vw, b), \vx, k)$ be an instance of the problem, and let us assume without loss of generality that~$\M(\vx) = 1$, as the other case is analogous. For each feature~$i$ of~$\vx$ we
define its importance~$s(i)$ as~$w_i$ if~$x_i = 1$ and~$-w_i$
otherwise. Intuitively,
$s$ represents how good it is to keep a certain feature in order to maintain the verdict of the model. We now assume that~$\vx$ and~$\vw$
have been sorted in decreasing order of score~$s$ (paying the cost of a sorting procedure) . For example, if originally
$\vw = (3, -5, -2)$ and~$\vx = (1,0,1)$, then after the sorting procedure we
have ~$\vw = (-5, 3, -2)$ and~$\vx = (0, 1, 1)$.
	 This sorting procedure has cost~$O(|\M|)$ as it is a classical problem of sorting strings whose total length add up to $\M$ and can be carried with a variant of Bucketsort~\cite{CLRS}. As a result, for every pair~$1 \leq i < j \leq n$ we have that~$s(i) \geq s(j)$. 
	Let~$k'$ be the largest integer no greater than~$k$ such that~$s(k') > 0$ and then define~$\vx'$ as the instance that differs from~$\vx$ exactly on the first~$k'$ features. We claim that~$\M(\vx') \neq \M(\vx)$ if and only if~$(\M, \vx, k)$ is a positive instance of \textsc{MinimumChangeRequired}. 
	The forward direction follows from the fact that~$k' \leq k$. 
	For the backward direction, assume that~$(\M, \vx, k)$ is a positive instance of \textsc{MinimumChangeRequired}. This implies that there is an instance~$\vy$ that differs from~$\vx$ in at most~$k$ features, and for which~$\M(\vy) =0$.  If~$\vy = \vx'$, then we are immediately done, so we can safely assume this is not the case.
	 
We then define, for any instance~$\vy$ of~$\M$ the function~$v(\vy) = \langle \vw, \vy \rangle$. Note that an instance~$\vy$ of~$\M$ is positive if and only if~$v(\vy) \geq - b$. 
Then, since we have that~$\M(\vy) = 0$, it holds that~$v(\vy) < -b$. We now claim that~$v(\vx') \leq v(\vy)$:	
	\begin{claim}
	For every instance~$\vy$ such that~$d(\vy, \vx) \leq k$ and~$\M(\vy) \neq \M(\vx)$, it must hold that~$v(\vx') \leq v(\vy)$.
	\end{claim}
	
	\begin{proof}

	For an instance~$\vz$, let us write~$C_\vz$ for the set of features for which~$\vz$ differs from~$\vx$. We then have on the one hand 
	\[
		v(\vx') = \sum_{i \in C_{\vx'} \setminus C_{\vy}}(1-x_i)w_i + \sum_{i \in C_{\vy} \cap C_{\vx'}}(1-x_i)w_i + \sum_{i \not\in C_{\vx'} \cup C_{\vy}}x_i w_i
		+ \sum_{i \in C_{\vy} \setminus C_{\vx'}}x_i w_i
	\]
	and on the other hand
	\[
		v(\vy) = \sum_{i \in C_{\vy} \setminus C_{\vx'}}(1-x_i)\vw_i + \sum_{i \in C_{\vy} \cap C_{\vx'}}(1-x_i)w_i + \sum_{i \not\in C_{\vx'} \cup C_{\vy}}x_i w_i + \sum_{i \in C_{\vx'} \setminus C_{\vy}}x_i w_i
	\]
	As the two middle terms are shared, we only need to prove that 
	\[
		\sum_{i \in C_{\vx'} \setminus C_{\vy}}(1-x_i)w_i + \sum_{i \in C_{\vy} \setminus C_{\vx'}}x_i w_i
 	\leq
 		\sum_{i \in C_{\vy} \setminus C_{\vx'}}(1-x_i)w_i 
 		+ \sum_{i \in C_{\vx'} \setminus C_{\vy}}x_i w_i
	\]
	which is equivalent to proving that
	\[
		\sum_{i \in C_{\vx'} \setminus C_{\vy}, x_i =0}w_i + \sum_{i \in C_{\vy} \setminus C_{\vx'}, x_i = 1}w_i
 	\leq
 		\sum_{i \in C_{\vy} \setminus C_{\vx'}, x_i = 0} w_i 
 		+ \sum_{i \in C_{\vx'} \setminus C_{\vy}, x_i = 1} w_i
	\]
	and by using the definition of importance, equivalent to
	\[
		\sum_{i \in C_{\vx'} \setminus C_{\vy}, x_i =0}-s(i) + \sum_{i \in C_{\vy} \setminus C_{\vx'}, x_i = 1}s(i)
 	\leq
 		\sum_{i \in C_{\vy} \setminus C_{\vx'}, x_i = 0} -s(i) 
 		+ \sum_{i \in C_{\vx'} \setminus C_{\vy}, x_i = 1} s(i)
	\]
	which can be rearranged into 
	\[
	\sum_{i \in C_{\vy}\setminus C_{\vx'}} s(i) \leq \sum_{i \in C_{\vx'}\setminus C_{\vy}} s(i)
	\]
	But this inequality must hold as~$C_\vx'$ is by definition the set~$C$ of features of size at most~$k$ that maximizes~$\sum_{i \in C} s(i)$.
	\end{proof}

Because of the claim, and the fact that~$v(\vy) < -b$ we conclude that~$v(\vx') < -b$, and thus~$\M(\vx') \neq \M(\vx)$. This concludes the backward direction, and thus, the fact that checking whether~$\M(\vx') \neq \M(\vx)$ is enough to solve the entire problem. Since checking this can be done in linear time, constructing~$\vx'$ is the most expensive part of the process, which can effectively be done in time~$O(|\M|)$. This concludes the proof of the lemma.
\end{proof}

\begin{lemma}
	The \textsc{MinimumChangeRequired} query is~$\NP$-complete for MLPs.
\end{lemma}

\begin{proof}
	Membership is easy to see, it is enough to non-deterministically guess an instance~$\vy$ and check that~$d(\vx, \vy) \leq k$ and~$\M(\vx) \neq \M(\vy)$.
	
	In order to prove hardness, we reduce from \textsc{Vertex Cover}.  Given an undirected graph~$G = (V,E)$ and an integer~$k$, the \textsc{Vertex Cover} problem consists in deciding whether there is a subset~$S \subseteq V$ of at most~$k$ vertices such that every edge of~$G$ touches a vertex in~$S$. Let~$(G = (V,E), k)$ be an instance of \textsc{Vertex Cover}, and let~$n$ denote~$|V|$. Based on~$G$, we build a formula~$\varphi_G$, where propositional variables correspond to vertices of~$G$.
	
	\[
	\varphi_G = \bigwedge_{(u, v) \in E} (x_u \lor x_v)
	\]
	
	It is clear that the satisfying assignments of~$\varphi_G$ correspond to the vertex covers of~$G$, and furthermore, that a satisfying assignment of Hamming weight~$k$ (number of variables assigned to~$1$) corresponds to a vertex cover of size~$k$.
	
	Moreover, we can safely assume that there is at least~$1$ edge in~$G$, as otherwise the instance would be trivial, and a constant size positive instance of MCR would finish the reduction. This implies in turn, that we can assume that assigning every variable to~$0$ does not satisfy~$\varphi_G$.
	
	We now build an MLP~$\M_\varphi$ from~$\varphi_G$, using Lemma~\ref{lem:circuits-to-MLPs}. We claim that the instance~$(\M_\varphi, 0^n, k)$ is a positive instance of \textsc{MinimumChangeRequired} if and only if~$(G,k)$ is a positive instance of \textsc{Vertex Cover}.
	
	Indeed,~$0^n$ is a negative instance of~$\M_\varphi$, as assigning  every variable to~$0$ does not satisfy~$\varphi_G$. Moreover a positive instance of weight~$k$ for~$\M_\varphi$ corresponds to a satisfying assignment of weight~$k$ for~$\varphi_G$, which in turn corresponds to a vertex cover of size~$k$ for~$G$. This is enough to conclude conclude the proof, recalling that both the construction of~$\varphi_G$ and~$\M_\varphi$ take polynomial time.
\end{proof}

%
%
\section{Proof of Proposition~\ref{prp:ksuff}}
\label{sec:proof-6}
In this section we prove Proposition~\ref{prp:ksuff}, whose statement we recall here:

\ksuff*

Again, we prove each claim separately.

\begin{lemma}
The \textsc{MinimumSufficientReason} query is $\np$-complete for FBDDs, and hardness holds already for decision trees.
\end{lemma}
\begin{proof}
Membership in NP is clear, it suffices to guess the instance $\vy$ and check both that it has less than $k$ defined components and that is a sufficient reason for $\vx$, which can be done thanks to Lemma~\ref{lem:CheckSufficientReason-fbdds}.
We will prove that hardness holds already for the particular case of decision trees, and when the input instance $\vx$ is positive. Hardness of this particular setting implies of course the hardness of the general problem. In order to do so, we will reduce from the problem of determining whether a directed acyclic graph has a dominating set of size at most $k$, which we abbreviate as \textsc{Dom-DAG}. 
Recall that in a directed graph $G = (V,E)$, a subset of vertices $D \subseteq V$ is said to be dominating if every vertex in $V \setminus D$ has an incoming edge from a vertex in $D$. 
The problem of \textsc{Dom-DAG} is shown to be NP-complete in \cite{King_2005}. 

An illustration of the reduction is presented in Figure~\ref{fig:2}. Let $\left( G = (V,E), k \right)$ be an instance of \textsc{Dom-DAG}, and let us define $n := |V|$.
We start by computing in polynomial time a topological ordering $\varphi = \varphi_1, \ldots, \varphi_{n}$ of $G$.  
Next, we will create an instance $(\mathcal{T}, \vx, k)$ of \textsc{$k$-SufficientReason} such that there is a sufficient reason of size at most $k$ for $\vx$ under the decision tree $\mathcal{T}$ if and only if $G$ has a dominating set of size at most $k$. We create the decision tree $\mathcal{T}$, of dimension $n$, in 2 steps.
\begin{enumerate}
	\item Create nodes $v_1, \ldots, v_{n}$, where node $v_i$ is labeled with $\varphi_i$ The node $v_n$ will be the root of~$\mathcal{T}$, and for $2 \leq i \leq n$, connect $v_i$ to $v_{i-1}$ with an edge labeled with $1$. Node $v_1$ is connected to a leaf labeled $\true$ through an edge labeled with $1$. We will denote the path created in this step as $\pi$.
	\item For every vertex $\varphi_i$ create a decision tree $\mathcal{T}_i$ equivalent to the boolean formula
	\[ 
		\mathcal{F}_i = 	\bigvee_{(\varphi_j, \varphi_i) \in E} \varphi_j 
	\]
	and create an edge from $v_i$ to the root of $\mathcal{T}_i$ labeled with $0$. If $\mathcal{F}_i$ happens to be the empty formula, $\mathcal{T}_i$ is defined as $\false$. 
	Note that the nodes introduced by this step are all naturally associated with vertices of $G$.
\end{enumerate}

\begin{figure}
	\begin{subfigure}{.25\textwidth}
	\centering
		\begin{tikzpicture}
		\begin{scope}[every node/.style={circle,thick,draw}]
			\node (1) at (0, 0) {1};
			\node[ultra thick] (2) at (-1, 0) {2};
			\node (3) at (1, 0) {3};
			\node (4) at (-0.5, -1) {4};
			\node[ultra thick] (5) at (0.5, -1) {5};
			\node (6) at (0, 1) {6};
		\end{scope}
	    \begin{scope}[every path/.style={->, thick}]
			\path (1) edge (6);
			\path (1) edge (3);
			\path (2) edge (1);
			\path (2) edge (6);
			\path (4) edge (2);
			\path (4) edge (1);
			\path (5) edge (1);
			\path (5) edge (3);
			\path (5) edge (4);
			\path (6) edge (3);
		\end{scope}
		\end{tikzpicture}
		\caption{Example of an input DAG. Nodes $2$ and $5$, corresponding to the minimum dominating set of $G$ are emphasized.}
	\end{subfigure}
	\begin{subfigure}{.25\textwidth}
		\centering
		\begin{tikzpicture}
		\begin{scope}[every node/.style={circle,thick,draw}]
			\node (5) at (0, -2) {5};
			\node (4) at (0, -1) {4};
			\node (2) at (0, 0) {2};
			\node (1) at (0, 1) {1};
			\node (6) at (0, 2) {6};
			\node (3) at (0, 3) {3};
		\end{scope}
	    \begin{scope}[every path/.style={->, thick}]
			\path (1) edge (6);
			\path (1) edge [bend left=30]  (3);
			\path (2) edge (1);
			\path (2) edge [bend right=30] (6);
			\path (4) edge (2);
			\path (4) edge [bend left=30]  (1);
			\path (5) edge [bend right=30] (1);
			\path (5) edge [bend left=30] (3);
			\path (5) edge (4);
			\path (6) edge (3);
		\end{scope}
		\end{tikzpicture}
		\caption{A topological ordering $\varphi$ of $G$.}
	\end{subfigure}
	\begin{subfigure}{.5\textwidth}
		\begin{tikzpicture}
		\begin{scope}[every node/.style={thick,draw}]	
			\node (3) at (0, 0) {3};
			\node (6) at (0.8, -0.6) {6};
			\node (1) at (1.6, -1.2) {1};
			\node [ultra thick] (2) at (2.4, -1.8) {2};
			\node (4) at (3.2, -2.4) {4};
			\node[ultra thick, text=black] (5) at (4, -3) {5};			
			\node[ultra thick, text=black] (s3_1) at (-0.8, -0.6) {5};
			\node[text=black] (s3_2) at (-1.6, -1.2) {1};
			\node[text=black] (s3_3) at (-2.4, -1.8) {6};
			
			\node[text=black] (s6_1) at (0, -1.2) {1};
			\node[ultra thick, text=black] (s6_2) at (-0.8, -1.8) {2};
			
			\node[text=black] (s1_1) at (0.8, -1.8) {5};
			\node[text=black] (s1_2) at (0, -2.4) {4};
			\node[ultra thick, text=black] (s1_3) at (-0.8, -3) {2};
			
			\node[text=black] (s2_1) at (1.6, -2.4) {4};
			
			\node[ultra thick, text=black] (s4_1) at (2.4, -3) {5};
		\end{scope}
		
		
		\begin{scope}[every path/.style={->, thick}]
			\path (3) edge (6);
			\path (6) edge (1);
			\path (1) edge (2);
			\path (2) edge (4);
			\path (4) edge (5);
			\path (3) edge (s3_1);
			\path (s3_1) edge (s3_2);
			\path (s3_2) edge (s3_3);
			\path (6) edge (s6_1);
			\path (s6_1) edge (s6_2);
			\path (1) edge (s1_1);
			\path (s1_1) edge (s1_2);
			\path (s1_2) edge (s1_3);
			\path (2) edge (s2_1);
			\path (4) edge (s4_1);
		\end{scope}
		\end{tikzpicture}
		\caption{Resulting decision tree $\mathcal{T}$. Edges to the left of a node are always labeled with~$0$, and edges to the right with~$1$. The leaves are not depicted for clarity, but: if a node has no right child in the picture, then its right child is~$\true$, and if it has no left child then its left child is~$\false$. Note that in every diagonal there is an emphasized node, which is either $2$ or $5$, implying the partial instance $(\bot,1,\bot,\bot,1,\bot)$ is a sufficient reason for the instance~$\vx = (1,1,1,1,1,1)$.}
	\end{subfigure}
	\caption{Illustration of the reduction from \textsc{Dom-DAG} to \textsc{$k$-SufficientReason} over decision trees, for an example graph of $6$ nodes.}
	\label{fig:2}
\end{figure}
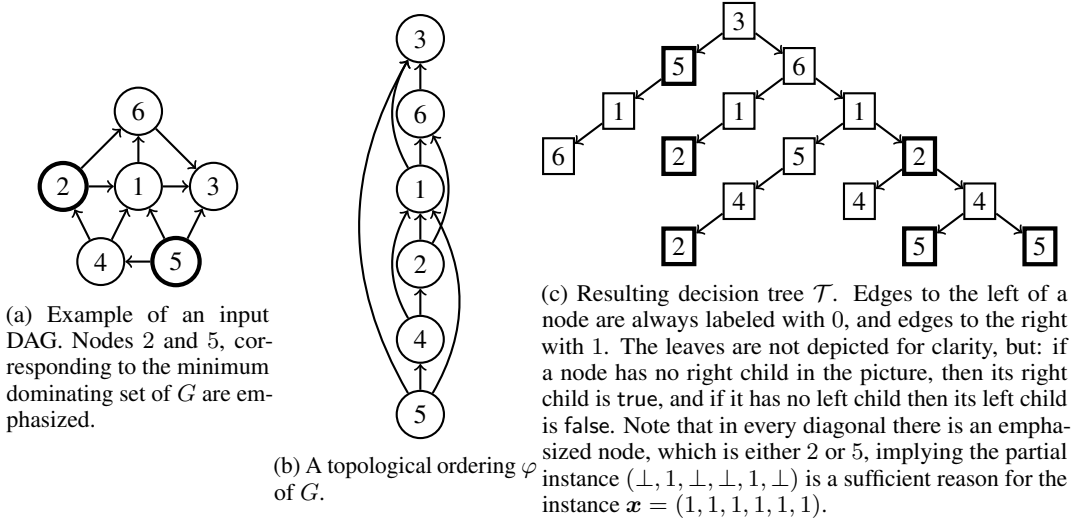

Step 2 takes polynomial time because boolean formulas in 1-DNF can easily be transformed into a decision tree in linear time.

We now check that $\mathcal{T}$ is a decision tree. Since~$\mathcal{T}$ has a tree structure, it is enough to check that for every path from the root to a leaf there are no two nodes on the path that have the same label (i.e., to check that~$\mathcal{T}$ is a valid FBDD). Note that any path from the root $v_n$ to a leaf goes first to a certain node $v_i$ in~$\pi$, from where it either takes an edge labeled with $0$, in case $i \neq 1$ or it simply goes to a leaf otherwise. In case $i = 1$, the path from the root goes exactly through $v_n, v_{n-1}, \ldots, v_1$, which all have different labels. 
In case $i \neq 1$, the path includes 
(i) nodes with labels $\varphi_n, \varphi_{n-1}, \ldots, \varphi_i$, and (ii) a subpath inside $\mathcal{T}_i$. It is clear that all the labels in (i) are different. And as by construction $\mathcal{T}_i$ is a decision tree, no two nodes inside (ii) can have the same label. It remains to check that no node in (i) can have the same label of a node in (ii). To see this, consider that all the vertices of $G$ associated to the nodes in (ii) have edges to $\varphi_i$ in $G$, and thus come before $\varphi_i$ in the topological order. But (i) is composed precisely by $\varphi_i$ and the nodes who come after it in the topological ordering, so (i) and (ii) have empty intersection. 	

Let $\vx = 1^n$ be the vector of $n$ ones. We claim that $(\mathcal{T},\vx, k)$ is a yes-instance of~\textsc{$k$-SufficientReason} if and only if $(G,k)$ is a yes-instance of~\textsc{Dom-DAG}.
  
\textbf{Forward direction.} Consider that there is a sufficient reason $\vy$ for $\vx$ under $\mathcal{T}$ of size at most $k$. As $\vx$ contains only $1$s, $\vy$ must contain only $1$s and~$\bot$s. Consider the set $S$ of components $i$ where $\vy_i = 1$. Recalling that every vertex of $G$ is canonically associated with a feature of $\mathcal{T}$, we will denote $D_S$ to the set of vertices of $G$ that are associated with the features in $S$. 
It is clear that $|D_S| \leq k$. We now prove that $D_S$ is a dominating set of $G$. First, in case $D_S = V$, we are trivially done. We know assume $D_S \neq V$. 
Consider a vertex~$v \in V \setminus D_S$, corresponding to~$\varphi_i$ in the topological ordering, and define~$\vz$ as the completion of~$\vy$ where the features~$\varphi_j$ such that~$j > i$, are set to~$1$, and all other features that are undefined by~$\vy$ are set to~$0$. 
By hypothesis,~$\vz$ must be a positive instance, and so its path on~$\mathcal{T}$ must end in a leaf labeled with~$\true$. 
Note that the path of~$\vz$ in~$\mathcal{T}$ necessarily takes the path~$\pi$ created in Step 1 of the construction, up to the node~$v_i$, and then enters its subtree~$\mathcal{T}_i$. 
Let~$t$ be the node of~$\mathcal{T}_i$ whose leaf labeled with~$\true$ ends the path of~$\vz$ in~$\mathcal{T}$, and~$\varphi_k$ its label and associated vertex in~$G$. As feature~$t$ is set to~$1$, we must have either~$\varphi_k \in D_S$ (in case~$t$ is~$1$ because of~$\vy$) or~$k > i$ (in case~$t$ is~$1$ by the construction of completion~$\vz$).
However, the second case is not actually possible, as if~$k > i$, that means~$v_k$ comes before~$v_i$ in path~$\pi$, and thus the path of~$\vz$ in~$\mathcal{T}$ passes through~$v_k$, which has label~$\varphi_k$ before passing through~$v_i$. But the path of~$\vz$ in~$\mathcal{T}$ passes through~$t$ before ending, which also has label~$\varphi_k$. This contradicts the already proven fact that~$\mathcal{T}$ is a decision tree. 
We can therefore assume that~$\varphi_k$ belongs to ~$D_S$. Then, as~$t$ is a node of~$\mathcal{T}_i$, there must be an edge~$(\varphi_k, \varphi_i)$ in~$E$ because of the way~$\mathcal{T}_i$ is constructed. But this means that vertex~$v \in V \setminus D_S$ has an edge coming from~$\varphi_k \in D_S$, and so~$v$ is effectively dominated by the set~$D_S$. As this holds for every~$v \in V \setminus D_S$, we conclude that~$D_S$ is indeed a dominating set of~$G$.

\textbf{Backward Direction.} Consider that there is a dominating set~$D \subseteq V$ of size at most~$k$. Let~$S_D$ be the set of features associated with~$D$. We claim that the partial instance~$\vy$ that has~$1$ in the features that belong to~$S_D$, and is undefined elsewhere, is a sufficient reason for~$\vx$, and by construction its size is at most~$k$. Consider an arbitrary completion~$\vz$ of~$\vy$, we need to show that~$\vz$ is a positive instance of~$\mathcal{T}$. For~$\vz$ not to be a positive instance, its path on~$\mathcal{T}$ would have to reach a leaf labeled with~$\false$. 
This can only happen by either taking the edge labeled with~$0$ from~$v_1$ (the last node in path~$\pi$ built in the construction), or inside a subtree~$\mathcal{T}_i$, corresponding to a node~$v_i$ whose associated feature in~$\vz$ is set to~$0$. We show that neither can happen. 
For the first case, every dominating set must include~$\varphi_1$, the vertex in~$G$ associated with~$v_1$, as it is the first element in the topological ordering of~$G$,  and thus it must has in-degree~$0$, which implies~$\varphi_1 \in D$. 
Therefore, it is not possible to take the edge labeled with~$0$ from~$v_1$. On the other hand, suppose the path of~$\vz$ in~$\mathcal{T}_i$ ends in a leaf labeled with~$\false$. 
Then, by construction of~$\mathcal{T}_i$, there is no vertex~$\varphi_j$ such that~$(\varphi_j, \varphi_i) \in E$ whose associated feature is set to~$1$ in~$\vz$. But as~$D$ is a dominating set, either there is indeed a~$\varphi_j \in D$ such that~$(\varphi_j, \varphi_i) \in E$ or~$\varphi_i \in D$. 
The first case is in direct contradiction with the previous statement, as~$\varphi_j \in D$ implies, by our construction of~$\vy$ that the feature associated with~$\varphi_j$ is set to~$1$.
The second case also creates a contradiction, as if $\varphi_i \in D$, then by construction $\vy$ would have a $1$ in the feature $v_i$ associated to $\varphi_i$, which contradicts the assumption of the path of $\vz$ entering $\mathcal{T}_i$. 

\end{proof}

\begin{lemma}
The \textsc{MinimumSufficientReason} query is in \ptime\ for perceptrons.
\end{lemma}

\begin{proof}
Let~($\mathcal{M} = (\vw, b),\vx,k)$ be an instance of the problem, and let~$d$ be the dimension of the perceptron.
We will assume without loss of generality that~$\M(\vx)=1$.  In this proof, what we call a \emph{minimum sufficient
reason for~$\vx$} is a sufficient reason for~$\vx$ that has the least number of
components being defined.
We show a greedy algorithm that computes a minimum
sufficient reason for~$\vx$ under~$\mathcal{M}$ in
time~$O(|\mathcal{M}|)$. For each feature~$i$ of~$\vx$ we
define its importance~$s(i)$ as~$w_i$ if~$x_i = 1$ and~$-w_i$
otherwise (just as we did in the proof of Lemma~\ref{lemma:mcr-perceptron}), and its \textit{penalty}~$p(i)$ as~$\min(0, w_i)$. Intuitively,
$s$ represents how good it is for a partial instance to be defined in a given
feature, and~$p$ represents the penalty or cost that a partial instance incurs
by not being defined in a given feature.  We now assume that~$\vx$ and~$\vw$
have been sorted in decreasing order of score~$s$. For example, if originally
$\vw = (3, -5, -2)$ and~$\vx = (1,0,1)$, then after the sorting procedure we
have ~$\vw = (-5, 3, -2)$ and~$\vx = (0,1,1)$. We now define a function~$\psi$
that takes any partial instance~$\vy$ as input and outputs the worst possible
value for a completion of~$\vy$:
	\[
	\psi(\vy) \coloneqq \min_{\vz : \, \vz \text { is a completion of } \vy} \langle \vw, \vz\rangle = \sum_{y_i \neq \bot} w_i y_i + \sum_{y_i = \bot} p(i). 
	\]
	
The second equality is easy to see based on the definition of the function $p$, and the definition of $\psi$ implies that $\psi(\vy) \geq -b$ exactly when $\vy$ is a sufficient reason. For $1 \leq l \leq d$, we define $\vy^l$ as the partial instance of $\vx$ such that $y^l_i$ is equal to $x_i$ if $i \leq l$ and to $ \bot$ otherwise. In simple terms, $\vy^l$ is the partial instance obtained by taking the first $l$ features of $\vx$; continuing our example with~$\vx = (0,1,1)$, we have for instance~$\vy^2 = (0,1,\bot)$.
Let $j$ be the minimum index such that $\psi(\vy^j) \geq  - b$.  Such an index always exists, because, since $\vx$ is a positive instance, taking $j = d$ is always a valid index. Note that $j$ can be computed in linear time.

  	We now prove that ($\dagger$) the partial instance $\vy^j$  is a minimum sufficient reason for $\vx$.
By definition we have that $\psi(\vy^j) \geq  - b$, so $\vy^j$ is indeed a
sufficient reason for $\vx$. We now need to show that~$\vy^j$ is minimum.
Assume, seeking for a contradiction, that there is a sufficient reason~$\vy'$
of~$\vx$ with strictly less components defined than $\vy^j$; clearly we can
assume without loss of generality that $\vy'$ has exactly $j-1$ components
defined.  We will now show that ($\star$) $\vy^{j-1}$ is a also a sufficient
reason for $\vx$, which is a contradiction since~$j$ was assumed to be the
minimal index such that~$\vy^j$ is a sufficient reason of~$\vx$, hence proving~($\dagger$). If~$\vy' = \vy^{j-1}$, we have that ($\star$) is
trivially true. Otherwise, and considering that~$\vy'$ and~$\vy^{j-1}$ have the same
size, and that~$\vy^{j-1}$ is defined exactly on the first~$j-1$ features,
there must be at least a pair of features~$(m, n)$, with~$m \leq j-1 < n$, such
that~$\vy^{j-1}$ is defined at feature~$m$ and~$\vy'$ is not, and on the other
hand~$\vy'$ is defined at feature~$n$ whereas~$\vy^{j-1}$ is not. In order to
finish the proof of ($\star$), we will prove a simpler claim that will
help us conclude.
	
	\begin{claim}
    Assume that there is a pair of features~$(m, n)$ with~$m\leq j-1<n$ such
that $y'_m = \bot, y^{j-1}_m \neq \bot$ and $y'_n \neq \bot, y^{j-1}_n
= \bot$, and let $\vy^*$ be the resulting partial instance that is equal
to~$\vy'$ except that $y^*_m := y^{j-1}_m$ and $y^*_n := \bot$. Then we
have that  $\psi(\vy^*) \geq \psi(\vy')$.
	\label{claim:psi}
	\end{claim}
\begin{proof}[Proof of Claim~\ref{claim:psi}]
By definition, $\psi(\vy^*) - \psi(\vy') = p(n) - p(m) + w_m y^{j-1}_m -
w_n y'_n =  p(n) - p(m) + w_m x_m - w_n x_n $. But because the features
in $y^{j-1}$ are sorted in decreasing order of score, it must hold that
$s(m) \geq s(n)$. Using this last inequality and reasoning by cases on the
values $x_m, x_n$ and on the signs of $w_m, w_n$, one can tediously check
that $\psi(\vy^*) - \psi(\vy') \geq 0$ and thus $\psi(\vy^*) \geq \psi(\vy')$.
\end{proof}

We now continue with the proof of ($\star$).  As a result of
Claim~\ref{claim:psi}, one can iteratively modify $\vy'$ until it becomes equal
to~$\vy^{j-1}$ in such a way that the value of $\psi$ is never decreased along
the process, implying therefore that $\psi(\vy^{j-1}) \geq \psi(\vy')$.
But~$\psi(\vy') \geq -b$, because $\vy'$ is assumed to be a sufficient reason,
hence we have that $\psi(\vy^{j-1}) \geq -b$, implying that $\vy^{j-1}$ is a
sufficient reason for $\vx$, and thus concluding the proof of ($\star$).
Therefore, ($\dagger$) is proven, and since~$\vy^j$ can clearly be
computed in polynomial time (in fact, the runtime of the whole procedure is dominated by the sorting subroutine, which again has cost~$O(|\M|)$ as it is a classical problem of sorting strings whose total length add up to $|\M|$ and can be carried with a variant of Bucketsort \cite{CLRS}), this finishes the proof of the lemma; indeed, we can output $\textsc{Yes}$ if~$j\leq k$ and $\textsc{No}$ otherwise.\qedhere
\end{proof}

\begin{lemma}
\label{lem:sigma2p}
The \textsc{MinimumSufficientReason} query is $\Sigma_2^p$-complete for MLPs.
\end{lemma}
\begin{proof}
Membership in~$\Sigma_2^p$ is clear, as one can non-deterministically guess the value of the $k$ features that would make for a sufficient reason, and then use an oracle in co-NP to verify that no completion of that guess has a different classification.
To show hardness, we will reduce from the problem \textsc{Shortest Implicant Core}, defined and proven
to be $\Sigma_2^p$-hard by Umans \cite[Theorem 1]{Umans2001}. First, we need a few
definitions in order to present this problem.  A formula in \emph{disjunctive
normal form} (DNF) is a Boolean formula of the form~$\varphi = t_1 \lor t_2
\lor \ldots \lor t_n$, where each \emph{term}~$t_i$ is a conjunction of
literals (a literal being a variable of the negation thereof). An
\emph{implicant} for~$\phi$ is a partial assignment of the variables of~$\phi$
such that any extension to a full assignment makes the formula evaluate
to~$\true$; note that we can equivalently see an implicant of~$\phi$ as  what
we call a sufficient reason of~$\phi$. For a partial assigment~$C$
of the variables and for a set of literals~$t$ (or conjunction of
literals~$t$), we write~$C \subseteq t$ when for every variable~$x$, if~$x\in
t$ then~$C(x)=1$ and if~$\lnot x \in t$ then~$C(x)=0$ and~$C(x)$ is undefined
otherwise.  An instance of \textsc{Shortest Implicant Core} then consists of a
DNF formula $\varphi = t_1 \lor t_2 \lor \ldots \lor t_n$, together with an
integer $k$. Such an instance is positive for \textsc{Shortest Implicant Core}
when there is an implicant $C$ for $\varphi$ such that~$C \subseteq
t_n$.\footnote{Note that, in order to keep our notation consistent, we use the
symbol $\subseteq$ where Umans uses $\supseteq$.} Note that the  \textsc{Shortest Implicant Core} is closer to the problem at hand than the general \textsc{Shortest Implicant} problem, as  (minimum) sufficient reasons of an instance $\vx$ can only induce literals according to $\vx$, in a similar fashion of implicants that can only induce literals according to the core $t_n$.

\paragraph{A reduction that does not work, and how to fix it on an example.} 
In order to convey the main intuition, we start by presenting a first tentative
of a reduction that does not work. Thanks to Lemma~\ref{lem:circuits-to-MLPs} we
know that it is possible to build an MLP $\mathcal{M}_{\varphi}$ equivalent to
$\varphi$. However, doing so directly creates a problem: we would need to find
a convenient instance $\vx$ such that $(\varphi, k) \in \textsc{Shortest
Implicant Core}$ if and only if $(\mathcal{M}_\varphi, \vx, k) \in
k\textsc{-SufficientReason}$. A natural idea is to consider
$t_n$ as a candidate for $\vx$, but the issue is that $t_n$ does not
necessarily include every variable. The next natural idea is to try with $\vx$ being
an arbitrary completion of $t_n$ (interpreting $t_n$ as the partial instance
that is uniquely defined by its satisfying assignment). This approach fails
because there could be a sufficient reason of size at most $k$ for such an $\vx$
that relies on features (variables) that are not in $t_n$.  We illustrate this
with an example for $n = 4$.
\[
	\varphi \coloneqq x_1 \overline{x_5} \lor \overline{x_2}\, \overline{x_6} \lor x_3 x_6
	\lor \overline{x_1}\, \overline{x_2} x_4 \lor \underbrace{x_1 x_3 x_5}_{t_4}
\]
While it can be checked that $(\varphi, 2) \not\in \textsc{Shortest Implicant
Core}$,
we have that $(\mathcal{M}_\varphi, (1, 0, 1, 0, 1, 1), 2)$ is in fact a positive
instance of \textsc{$k$-SufficientReason}, as the partial instance that assigns
$1$ to $x_3$ and $x_6$ and is undefined for the rest of the features, is a
sufficient reason of size $2$ for $\vx$. The issue is that we are allowing
$x_6$ to be part of the sufficient reason for $\vx$ even though~$x_6 \not \in t_4$.
We can avoid this from happening by splitting each variable that is not in
$t_n$, such as $x_6$, into $k+1$ variables, in such a way that defining the value of
$x_6$ would force us to define the value of all the $k+1$ variables, which is of
course unaffordable. Continuing with the example, we build the formula $\varphi'$ as
follows:

\[
	\varphi' \coloneqq  \bigwedge_{i=1}^3 \left( x_1 \overline{x_5} \lor \overline{x_2^i}\, \overline{x_6^i} \lor x_3 x_6^i
	\lor \overline{x_1}\, \overline{x_2^i} x_4^i \lor x_1 x_3 x_5 \right)
\]

Now we can simply take $(\mathcal{M}_{\varphi'}, \vx, 1)$ where $\vx$ is an arbitrary completion of $t_4$ over the new set of variables, for example, one that assigns~$1$ to the features~$1,3$ and~$5$, and~$0$ to all other features (variables).
Note that~$\varphi'$ is not a DNF anymore, but this is not a problem, since we only need to compute~$\M_{\varphi'}$.
It is then easy to check that this instance is equivalent to the original input instance. 

\paragraph{The reduction.} 
We now present the correct reduction and prove that it works.  Let~$(\varphi,
k)$ be an instance of \textsc{Shortest Implicant Core}. Let~$X_{c}$ be the set
of variables that are not mentioned in~$t_n$. We split every variable~$x_j \in
X_{c}$ into~$k+1$ variables~$x_j^1, \ldots x_j^{k+1}$ and for each~$i \in \{1,
\ldots, k+1\}$ we build~$\varphi^{(i)}$ by replacing every occurrence of a
variable~$x_j$, that belongs to~$X_{c}$, by~$x_j^i$. Finally we define
$\varphi'$ as the conjunction of all the~$\varphi^{(i)}$. That is,

\begin{align}
	\varphi^{(i)} &\coloneqq \varphi [x_j \to x_j^i \text {, for all } x_j \in X_c]\\
	\varphi' &\coloneqq \bigwedge_{i=1}^{k+1} \varphi^{(i)}
\end{align}

Observe that any meaningful instance of \textsc{Shortest Implicant Core} has~$k
< |t_n|$, so we can safely assume that~$k$ is given in unary, making this
construction polynomial.

We then use Lemma~\ref{lem:circuits-to-MLPs} to build an MLP
$\mathcal{M}_{\varphi'}$ from~$\varphi'$, in polynomial time. The features of
this model correspond naturally to the variables of~$\varphi'$, and thus we
refer to both features and variables without distinction. Let~$\vy$ be the
instance that assigns~$1$ to every variable that appears as a positive literal
in~$t_n$, and~$0$ to all other variables. We claim that~$(\varphi, k) \in
\textsc{Shortest Implicant Core}$ if and only if~$(\mathcal{M}_{\varphi'}, \vx, k) \in
k\textsc{-SufficientReason}$. For the forward direction, if there is an
implicant~$C \subseteq t_n$ of~$\varphi$, of size at most~$k$, then we claim
that~$C$ is also an implicant of each~$\varphi^{(i)}$. This follows from the
fact that every assignment~$\sigma$ that is consistent with~$C$ and satisfies
$\varphi$, has a related assignment~$\sigma^i$, that for every variable~$x_j
\in X_c$ assigns~$\sigma^i(x^i_j) = \sigma(x_j)$, and that is equal to~$\sigma$
for every~$x_j \not \in X_c$. It is clear that~$\sigma^i(\varphi^{(i)}) =
\sigma(\varphi)$, which concludes the claim. As~$C$ is an implicant of each
$\varphi^{(i)}$, it must also be an implicant of~$\varphi'$. Then, as
$\mathcal{M}_{\varphi'}$ is equivalent to~$\varphi'$ (as Boolean functions) by
construction, and~$\vx$ is consistent with~$C$ because it is consistent with
$t_n$, it follows that the partial instance that is induced by~$C$ is a
sufficient reason for~$\vx$ under~$\mathcal{M}_{\varphi'}$. For the backward
direction, assume there is a sufficient reason~$\vy$ for~$\vx$ under
$\mathcal{M}_{\varphi'}$, whose size is at most~$k$, and let~$C'$ be its
associated implicant for~$\varphi'$. We cannot say yet that~$C'$ is a proper
candidate for being an implicant core of~$\varphi$, as~$C'$ could contain
variables not mentioned by~$t_n$. Let us define~$X'_c$ to be the set of
variables of~$\varphi'$ that are not present in~$t_n$. Intuitively, as there
are~$k+1$ copies of each variable of~$X'_c$ in~$\varphi'$, no valuation of a
variable in~$X'_c$, for the formula~$\varphi$, can be forced by a sufficient
reason of size at most~$k$. We will prove this idea in the following claim,
allowing us to build an implicant~$C$ for which we can assure~$C \subseteq
t_n$.

\begin{claim}
Assume that there is an implicant~$C'$ of size at most~$k$ for~$\varphi'$, and
let~$C$ be the partial valuation that sets every variable $x$ that appear
in~$t_n$ and that is defined by~$C'$ to~$C'(x)$, and that leaves every other variable
undefined.
Then~$C'$ is an implicant of size at most~$k$ for~$\varphi$.
\label{claim:irredundantImplicant}
\end{claim}

\begin{proof}
	The set~$X'_c$ can be expressed as the union of~$k+1$ disjoint sets of variables, namely~$X^{1}_c, \ldots, X^{k+1}_c$, where~$X^{i}_c$ contains all variables of the form~$x^i_j$.
Since~$C'$ contains at most~$k$ literals, and there are~$k+1$ disjoint
sets~$X^{i}_c$, there must exist an index~$l$ such that~$X^{l}_c \cap C' =
\varnothing$. But then this implies that~$C$  is an implicant
of~$\varphi^{(l)}$. But~$\varphi^{(l)}$ is equivalent to~$\varphi$ up to
renaming of the variables that are not present in~$C$, therefore, the fact
that~$C$ is an implicant of~$\varphi^{(l)}$ implies that~$C$ must be an
implicant of~$\varphi$ as well.
\end{proof}

By using Claim~\ref{claim:irredundantImplicant} we get that~$C$ is an implicant
of~$\varphi$.
But we have that $C \subseteq t_n$, which is enough to conclude
that~$(\varphi, k) \in \textsc{Shortest Implicant Core}$ and finishes the
proof of Lemma~\ref{lem:sigma2p}.
\end{proof}

%
%
\section{Proof of Proposition~\ref{prp:checksuff}}
\label{sec:proof-7}
We now prove Proposition~\ref{prp:checksuff}, whose statement we recall here:

\checksuff*

We prove each claim separately.

\begin{lemma}
The query \textsc{CheckSufficientReason} can be solved in linear time for FBDDs.
\label{lem:CheckSufficientReason-fbdds}
\end{lemma}
\begin{proof}
Let~($\mathcal{M},\vx,\vy)$ be an instance of the problem, with~$\M$ being an FBDD.  
We first check that~$\vx$ is a completion of~$\vy$, which can clearly be done in linear time.
We the define
$\mathcal{M}'$ as the resulting FBDD from the following procedure: (i) For
every internal node in $\mathcal{M}$ with label $i$, delete its outgoing edge
with label $0$ if $\vy_i = 1$ and its outgoing edge with label $1$ if $\vy_i =
0$. We note here that~$\mathcal{M}'$ is not a well defined FBDDs, since some
internal nodes may have only one outgoing edge: more precisely, the
value~$\mathcal{M}(\vx')\in \{0,1\}$ is well defined for every instance~$\vx'$
that is a completion of~$\vy$, and is not defined for an instance~$\vx'$ that is
not a completion of~$\vy$.  To check whether~$\vy$ is a sufficient reason, we
can then simply check that every leaf that is reachable from the root
in~$\mathcal{M}'$ is labeled the same (either~$\true$ or~$\false$). This can
clearly be done in linear time by standard graph algorithms.
\end{proof}

\begin{lemma}
\label{lem:CheckSufficientReason-perceptrons}
The query \textsc{CheckSufficientReason} can be solved in linear time for perceptrons.
\end{lemma}
\begin{proof}
Let~$(\mathcal{M} = (\vw, b),\vx,\vy)$ be an instance of the problem. We first check in linear time that~$\vx$ is a completion of~$\vy$.
We then get rid of the components that are
defined by~$\vy$, as follows.  We define:
\begin{itemize}
    \item $A \coloneqq \sum_{y_i \neq \bot} y_i w_i$;
    \item $\vw' \coloneqq (w_i \mid y_i = \bot)$; and
    \item $b' \coloneqq b+A$;
\end{itemize}
 and let~$\mathcal{M}'$ be the perceptron~$(\vw',b')$. Notice that the
dimension of~$\mathcal{M}'$ is equal to the number of undefined components
of~$\vy$; we denote this number by~$m$. It is then clear that~$\vy$ is a
sufficient reason of~$\vx$ under~$\mathcal{M}$ if and only if every instance
of~$\mathcal{M}'$ is labeled the same. We can check this as follows.  Let~$J_1$
be the minimum possible value of $\langle \vw', \vx' \rangle$ (for~$\vx' \in
\{0,1\}^m$); $J_1$ can clearly be computed in linear time by setting~$x'_i = 0$
if~$w'_i \geq 0$ and~$x'_i = 1$ otherwise. Similarly we can compute the
maximal possible value~$J_2$ of $\langle \vw', \vx' \rangle$.  Then, every instance
of~$\mathcal{M}'$ is labeled the same if and only if it is not the case that~$J_1
< -b'$ and~$J_2 \geq -b'$, thus concluding the proof.\qedhere
\end{proof}

\begin{lemma}
The query \textsc{CheckSufficientReason} is $\conp$-complete for MLPs.
\end{lemma}
\begin{proof}
We first show membership in co-NP. 
Let~$(\M,\vx,\vy)$ be an instance of the problem.
Then~$\vy$ is a sufficient reason of~$\vx$ under~$\mathcal{M}$ if
and only if all the completions of~$\vy$ are labeled the same as~$\vx$. This
can clearly be checked in co-NP.

In order to prove hardness we reduce from TAUT, the problem of checking whether
an arbitrary boolean formula is a satisfied by all possible assignments of its
variables. This problem is known to be complete for $\conp$. Let $\mathcal{F}$ be
an arbitrary boolean formula. We use Lemma~\ref{lem:circuits-to-MLPs} to build an
equivalent MLP $\mathcal{M}$ in polynomial time (with the features of
$\mathcal{M}$ corresponding to the variables of~$\mathcal{F}$). Then
$\mathcal{F}$ is a tautology if and only if all completions of the partial
instance $\vy = \bot^n$ are positive instances of $\mathcal{M}$. First, we
construct an arbitrary instance~$\vx$ (for instance, the one with
all the features being~$0$), and we reject if~$\mathcal{M}(\vx)=0$. Then, we
accept if~$\vy$ is a sufficient reason of~$\vx$ under~$\mathcal{M}$, and we reject
otherwise. This concludes the reduction. 
\end{proof}

%
%
\section{Proof of Proposition~\ref{prp:counting}}
\label{sec:proof-8}
We prove Proposition~\ref{prp:counting}, whose statement we recall here:

\counting*

As we said in the main text, the first claim follows almost directly from the
definition of FBDDs; see~\cite{wegener2004bdds} for instance.  For the second claim, we will
rely on the \shp-hardness of the counting problem~\#Knapsack, as defined next:

\begin{definition}
    An input of the problem \#Knapsack consists of natural numbers
$s_1,\ldots,s_n, k \in \mathbb{N}$ (given in binary). The output is the number of subsets~$S
\subseteq \{1,\ldots,n\}$ such that $\sum_{i \in S} s_i \leq k$.
\end{definition}

The problem \#Knapsack is well known to be \shp-complete. Since we were not
able to find a proper reference for this fact, we prove it here by using
the~\#P-hardness of the problem~\#SubsetSum.  An input of the problem
\#SubsetSum consists of natural numbers $s_1,\ldots,s_n, k \in \mathbb{N}$, and
the output is the number of subsets~$S \subseteq \{1,\ldots,n\}$ such that
$\sum_{i \in S} s_i = k$.  The problem~\#SubsetSum is shown to be~\shp-complete
in~\cite[Theorem 4]{berbeglia2009counting}. From this we can deduce:

\begin{lemma}[Folklore]
    The problem~\#Knapsack is~\#P-complete.
\end{lemma}
\begin{proof}
Membership in \shp~is trivial. We prove hardness by polynomial-time reduction
from~\#SubsetSum. Let~$(s_1,\ldots,s_n, k) \in \mathbb{N}^{n+1}$ be an input
to~\#SubsetSum.  It is clear that $\text{\#SubsetSum}(s_1,\ldots,s_n, 0) =
\text{\#Knapsack}(s_1,\ldots,s_n, 0)$, and that for~$k \geq 1$ we have
$\text{\#SubsetSum}(s_1,\ldots,s_n, k) = \text{\#Knapsack}(s_1,\ldots,s_n, k) -
\text{\#Knapsack}(s_1,\ldots,s_n, k-1)$, thus establishing the reduction.
\end{proof}

We can now show the second claim of Proposition~\ref{prp:counting}.

\begin{lemma}
The query \textsc{CountCompletions} is \shp-complete for perceptrons.
\end{lemma}
\begin{proof}
Membership in \shp~is trivial. We show hardness by polynomial-time reduction
from \#Knapsack. Let~$(s_1,\ldots,s_n, k)$ be an input of \#Knapsack.
Let~$\mathcal{M}$ be the perceptron with weights~$s_1,\ldots,s_n$ and
bias~$-(k+1)$. Remember that we consider only perceptrons that use the step
activation function, so that an instance~$\vx \in \{0,1\}^n$ is positive
for~$\mathcal{M}$ if and only if~$\sum_{i=1}^n \vx_i s_i -(k+1) \geq 0$.  It is
then clear that $\text{\#Knapsack}(s_1,\ldots,s_n, k) = 2^n -
\textsc{CountPositiveCompletions}(\mathcal{M},\bot^n)$, thus establishing the
reduction.
\end{proof}

Finally, the third claim of Proposition~\ref{prp:counting} simply comes from the fact that MLPs can simulate arbitrary Boolean formulas (Lemma~\ref{lem:circuits-to-MLPs}), and the fact that counting the number of satisfying assignments of a Boolean formula (\#SAT) is \shp-complete.

%
%
\section{Proof of Proposition~\ref{prp:pseudo}}
\label{sec:proof-9}
We now prove Proposition~\ref{prp:pseudo}, that is:

\pseudo*

The first part of the proof is to show how to transform in polynomial time and arbitrary instance of~\textsc{CountPositiveCompletions} for perceptrons 
(with the weights and bias being integers given in unary) 
into an instance of~\#Knapsack that has the same number of solutions.

\begin{lemma}
\label{lem:perceptron-to-knap}
Let $\mathcal{M}=(\vw , b)$ be a perceptron having at least one positive
instance, with the weights and bias being integers given in unary,
 and let~$\vx$ be a partial instance.  We can build in polynomial time an
input~$(s_1,\ldots,s_m,k)\in \mathbb{N}^{m+1}$ of~\#Knapsack such
that~$\textsc{CountPositiveCompletions}(\mathcal{M},\vx)=\text{\#Knapsack}(s_1,\ldots,s_m,k)$,
with~$s_1,\ldots,s_m,k$ written in unary (i.e., their value is polynomial
in the input size).  \end{lemma}

\begin{proof}
The first step is to get rid of the components that are defined by~$\vx$, like
we did in Lemma~\ref{lem:CheckSufficientReason-perceptrons}. Define 
\begin{itemize}
    \item $A \coloneqq \sum_{x_i \neq \bot} x_i w_i$;
    \item $\vw' \coloneqq (w_i \mid x_i = \bot)$; and
    \item $b' \coloneqq b+A$;
\end{itemize}
 and let~$\mathcal{M}'$ be the perceptron~$(\vw',b')$. Notice that the
dimension of~$\mathcal{M}'$ is equal to the number of undefined components
of~$\vx$; let us write~$m$ this number.  It is then clear that
$\textsc{CountPositiveCompletions}(\mathcal{M},\vx)$ is equal to the number of
positive instances of~$\mathcal{M}'$, that is, of instances~$\vx' \in
\{0,1\}^m$ that satisfy \begin{equation}\label{eq:bla} \langle \vw', \vx' \rangle + b' \geq
0\end{equation}
Now, let~$J$ be the maximum possible value of~$\langle \vw', \vx' \rangle$; $J$ can
clearly be computed in linear time by setting~$x'_i = 1$ if~$w'_i \geq 0$
and~$\vx'_i = 0$ otherwise.  We then claim that the number of solutions to
Equation~\ref{eq:bla} is equal to the number of solutions of
\begin{equation}\label{eq:knap}\langle\vs,  \vx' \rangle \leq k,\end{equation}
where~$s_i \coloneqq |w'_i|$ for~$1\leq i \leq m$ and~$k \coloneqq J+b'$.
Indeed, consider the function~$h:\{0,1\}^m \to \{0,1\}^m$ defined componentwise
by~$h(x'_i) \coloneqq x'_i$ if~$w'_i < 0$ and $h(x'_i) \coloneqq 1 -
x'_i$ otherwise.  Then~$h$ is a bijection, and we will show that 
for
any~$\vx' \in \{0,1\}^m$, we have that~$\vx'$ satisfies Equation~\ref{eq:bla}
if and only if~$h(\vx')$ satisfies Equation~\ref{eq:knap}, from which our claim follows. In order to see this, consider that

\begin{align}
	(3) \iff \sum_{i} w'_i x'_i \geq -b' & \iff  \sum_{w_i \geq 0}{w'_i x'_i}  +  \sum_{w_i < 0}{w'_i x'_i} \geq -b' \\
	& \iff \sum_{w_i \geq 0}{|w'_i| x'_i}  -  \sum_{w_i < 0}{|w'_i| x'_i} \geq -b' \\
	& \iff \sum_{w_i < 0}{|w'_i| x'_i}  -  \sum_{w_i \geq 0}{|w'_i| x'_i} \leq b' \\ 
\end{align} 

On the other hand, we have
\begin{align}
	h(\vx') \text{ satisfies } (4) &\iff \sum_{i}{|w'_i| h(x'_i)} \leq J + b'\\ 
    & \iff  \sum_{w_i < 0}{|w'_i| x'_i} + \sum_{w_i \geq 0}{|w'_i| (1-x'_i)} \leq \sum_{w_i \geq 0}{|w'_i|} + b'\\
    & \iff (7)
\end{align}

Last, let us observe that we have~$k\geq 0$, as otherwise~$\mathcal{M}$
would not have any positive instance. Therefore $(s_1,\ldots,s_m,k)$ is a valid
input of~\#Knapsack, which concludes the proof.
\end{proof}

We can now easily
combine Lemma~\ref{lem:perceptron-to-knap} together with a well-known dynamic
programming algorithm solving~\#Knaspsack in pseudo-polynomial time.

\begin{proof}[Proof of Proposition~\ref{prp:pseudo}.]
Let $\mathcal{M}=(\vw , b)$ be a perceptron, with the weights and bias being
integers given in unary, and let~$\vx$ be a partial instance.  First, we check
that the maximal value of~$\langle \vx, \vw \rangle$ is greater than~$-b$, as
otherwise~$\mathcal{M}$ has no positive instance and we can simply return~$0$.
We then use Lemma~\ref{lem:perceptron-to-knap} to build in polynomial time an
instance $(s_1,\ldots,s_m,k)\in \mathbb{N}^{m+1}$ of~\#Knapsack such
that~$\textsc{CountPositiveCompletions}(\mathcal{M},\vx)=\text{\#Knapsack}(s_1,\ldots,s_m,k)$,
and with~$s_1,\ldots,s_m,k$ being written in unary (i.e., their value is
polynomial in the input size).  We can then
compute~$\text{\#Knapsack}(s_1,\ldots,s_m,k)$ by dynamic programming as
follows.  For~$i\in \{1,\ldots,m\}$ and~$C \in \mathbb{N}$, define the
quantity~$\mathrm{DP}[i][C] \coloneqq |\{S \subseteq \{1,..,i\} | \sum_{j \in
S} \vs_j \leq C \}|$. We wish to compute~$\mathrm{DP}[m][k]$. We can do so by
computing~$\mathrm{DP}[i][C]$ for $i\in \{1,\ldots,m\}$ and~$C \in
\{0,\ldots,k\}$, using the relation $\mathrm{DP}[i+1][C] = \mathrm{DP}[i][C] +
\mathrm{DP}[i][C-\vs_{i+1}]$, and starting with the convention
that~$\mathrm{DP}[0][a]=0$ for all~$a < 0$ and that $\mathrm{DP}[0][a]=1$ for
all~$a \geq 0$. It is clear that the whole procedure can be done in polynomial
time.
\end{proof}

%
%
\section{Proof of Proposition~\ref{prp:approx}}
\label{sec:proof-10}
We prove in this section Proposition~\ref{prp:approx}, whose statement we recall here:

\approx*

The fact that the query has no FPRAS for MLPs is because MLPs can efficiently simulate Boolean formulas (Lemma~\ref{lem:circuits-to-MLPs}), and it is well known
that the problem \#SAT (of counting the number of satisfying assignments of a Boolean formula) has no FPRAS unless~$\np = \mathrm{RP}$. Hence we only need to prove our claim concerning perceptrons.

\begin{proof}[Proof of Proposition~\ref{prp:approx} for perceptrons.]
We can assume without loss of generality that the weights and bias are integers,
as we can simply multiply every rational by the lowest common denominator (note
that the bit lenght of the lowest common denominator is polynomial, and that it
can be computed in polynomial time\footnote{We need to compute the least common multiple (lcm) of a set of integers $a_1, \ldots, a_n$. Indeed, it is easy to check that $lcm(a_1, \ldots, a_n) = lcm(lcm(a_1, \ldots, a_{n-1}), a_n)$, which reduces inductively the problem to computing the lcm of two numbers in polynomial time. It is also easy to check that $lcm(a_1, a_2) = \frac{a_1 a_2}{gcd(a_1, a_2)}$, where $gcd(a_1, a_2)$ is the greatest common divisor of $a_1$ and $a_2$. As  multiplication can clearly be carried in polynomial time, and Euclid's algorithm allows computing the $gcd$ function in polynomial time, we are done.}).  We then
transform the perceptron and partial instance to an input of~\#Knapsack with
the right number of solutions using Lemma~\ref{lem:perceptron-to-knap}, by
observing that the construction also takes polynomial time when the input
weights are given in binary (and by considering that the $s_1,\ldots,s_m,k$ are
also computed in binary).  We can then apply an FPTAS to this~\#Knapsack
instance, as shown in~\cite{gopalan2011fptas,rizzi2019faster}. 
\end{proof}

%
%
\section{Background in parameterized complexity}
\label{sec:p-background}
In this section we present the notions from parameterized complexity that we will need to prove Proposition~\ref{prp:mlpt}.

A \emph{parameterized problem} is a language~$L \subseteq \Sigma^* \times \mathbb{N}$, where~$\Sigma$ is a finite alphabet. For each element~$(x, k)$ of a parameterized problem, the second component is called the \textit{parameter} of the problem. A parameterized problem is said to be~\emph{fixed parameter tractable} (FPT) if the question of whether~$(x, k)$ belongs to~$L$ can be decided in time~$f(k) \cdot |x|^{O(1)}$, where~$f$ is a computable function.

The~$\mathrm{FPT}$ class, as well as the other classes we will introduce in this paper, are closed under a particular kind of reductions. A mapping~$\phi: \Sigma^* \times \mathbb{N} \to \Sigma^* \times \mathbb{N}$ between instances of a parameterized problem~$A$ to instances of a parameterized problem~$B$ is said to be an \emph{fpt-reduction} if and only if
\begin{itemize}
	\item~$(x,k)$ is a yes-instance of~$A \iff$~$\phi(x,k)$ is a yes-instance of~$B$.
	\item~$\phi(x,k)$ can be computed in time~$|x|^{O(1)} \cdot f(k)$;
	\item~There exists a computable function $g$ such that $k' \leq g(k)$, where $k'$ is the parameter of~$\phi(x,k)$.
	\end{itemize}

We define the complexity classes that are relevant for this article in terms of circuits. Recall that a circuit is a rooted directed acyclic graph where nodes of in-degree~$0$ are called \emph{input gates}, and that the root of the circuit is called the \emph{output gate}. Internal gates can be either \textsc{Or}, \textsc{And}, or \textsc{Not} gates. All \textsc{Not} nodes have in-degree~$1$. Nodes of types \textsc{And} and \textsc{Or} can either have in-degree at most~$2$, in which case they are said to be \textit{small} gates, or in-degree bigger than~$2$, in which case they are said to be \textit{large} gates. The \textit{depth} of a circuit is defined as the length (number of edges) of the longest path from any input node to the output node. The \textit{weft} of a circuit is defined as the maximum amount of large gates in any path from an input node to the output node. An \emph{assignment} of a circuit~$C$ is a function from the set of input gates in~$C$ to~$\{ 0, 1\}$. The weight of an assignment is defined as the number of input gates that are assigned~$1$. Assignments of a circuit naturally induce a value for each gate of the circuit, computed according to the label of the gate. We say an assignment \emph{satisfies} a circuit if the value of the output gate is~$1$ under that assignment.

The main classes we deal with are those composing the~$\W$-hierarchy and the~$\W(\Maj)$- hierarchy, a variant proposed by Fellows et al. \cite{Fellows}. These complexity classes can be defined upon the \textsc{Weighted Circuit Satisfiability} problem, parameterized by specific classes~$\mathcal{C}$ of circuits, as defined below.

\begin{center}
\fbox{\begin{tabular}{rl}
Problem: & \textsc{Weighted Circuit Satisfiability($\mathcal{C}$)}, abbreviated \textsc{WCS($\mathcal{C}$)} \\
Input: & A circuit~$C \in \mathcal{C}$\\
Parameter: & An integer~$k$\\
Output: & \textsc{Yes}, if there is a satisfying assignment of weight exactly~$k$ for~$C$, \\
& and \textsc{No} otherwise.
\end{tabular}}
\end{center}

We consider two restricted classes of circuits. First,~$C_{t,d}$, the class of circuits using the connectives \textsc{And, Or, Not} that have weft at most~$t$ and depth at most~$d$. On the other hand, we consider~$M_{t,d}$, the class of circuits that use (only) the \textsc{Majority} connective (that is satisfied exactly when more than half of its inputs are true), have weft at most~$t$ and depth at most~$d$. 
In the case of majority gates, we allow multiple parallel edges. Observe that, even though his is not useful for circuits with $(\textsc{Or}, \textsc{And}, \textsc{Not})$-gates, it allows circuits majority gates to receive multiple times a same input.
In the case of majority gates, a gate is said to be small if its fan-in is at most~$3$.

We can then define each class~$\W[t]$ (resp., $\W(\Maj)[t]$) as the set of parameterized problems that can be fpt-reduced to~$\textsc{WCS}(C_{t, d})$ (resp., $\textsc{WCS}(M_{t,d})$) for some constant~$d$.
Note that the notion of \emph{can be fpt-reduced} is transitive, and thus the classes $\W[t]$ and $\W(\Maj)[t]$ are closed under fpt-reductions.	
As usual, a parameterized problem~$A$ is then said to be~$\W[t]$-hard (resp., $\W(\Maj)[t]$-hard) when every parameterized problem in~$\W[t]$ (resp., $\W(\Maj)[t]$) can be fpt-reduced to~$A$.

%
%
\section{Proof of Proposition~\ref{prp:mlpt}}
\label{sec:proof-11}
In this section we prove Proposition~\ref{prp:mlpt}, that is:

\mlpt*

We first explain what are rMLPs, then sketch the proof, and then proceed with the proof.

Given an MLP~$\M$, with the dimension of the layers being~$d_0,\ldots,d_k$, we define its \emph{graph size} as~$N := \sum_{i=0}^k d_i$. We say an MLP with graph size~$N$ is restricted (abbreviated as rMLP) if each of its weights and biases can be represented as a decimal number with at most~$O(\log (N))$ digits. More precisely, represented as~$\sum_{i = -K}^K a_i 10^i$, for integers~$0 \leq a_i \leq 9$ and~$K \in O(\log N)$. Note that all numbers expressible in this way are also expressible by fractions, where the numerator is an arbitrary integer bounded by a polynomial in~$N$, and the denominator is a power of~$10$ whose value is bounded as well by a polynomial in~$N$.

We now explicit a family of parameterized problems indexed by an integer~$t \geq 1$.

\begin{center}
\fbox{\begin{tabular}{rl}
Problem: & \textsc{$t$-MinimumChangeRequired}, abbreviated \textsc{$t$-MCR} \\
Input: & An rMLP~$\M$ with at most~$t$ layers, an instance~$\vx$\\
Parameter: & An integer~$k$\\
Output: & \textsc{Yes}, if there exists an instance~$\vy$ with~$\dist(\vx,\vy) \leq k$ \\ 
& and~$\M(\vx)\neq\M(\vy)$, and \textsc{No} otherwise

\end{tabular}}
\end{center}

We rewrite the statement of Proposition~\ref{prp:mlpt} with this explicit notation.

\begin{proposition*}[\textbf{Restatement of Proposition~\ref{prp:mlpt}}]
For every~$t\geq 1$, the~$(3t+3)$-\textup{MCR} problem is~$\W(\Maj)[t]$-hard and is contained in~$\W(\Maj)[3t+7]$.
\end{proposition*}

As the proof of Proposition~\ref{prp:mlpt} is quite involved, we first present a proof sketch that summarizes the process.

\textbf{Hardness.} 
We prove hardness in Section~\ref{subsec:p-hardness}.  Showing that a parameterized
problem~$A$ is~$\W[t]$-hard (resp.,~$\W(\Maj)[t]$-hard) is usually complicated
since, by directly using the definition, one would have to show that for every
fixed~$d\in \mathbb{N}$, there exists an fpt-reduction~$f_d$
from~$\textsc{WCS}(C_{t,d})$ (resp., from~$\textsc{WCS}(M_{t,d})$) to~$A$.
Instead, it is usually more convenient to prove first some form of
\emph{normalization theorem} stating that a particular class of circuits, for
which one knows the value of $d$, is already hard for $\W[t]$ (or
$\W(\Maj)[t]$).\footnote{Useful normalization theorems for the $\W$-hierarchy are proved
in the work of Downey, Fellows and Regan \cite{Downey1995, Downey1998}, or Buss and Islam.
\cite{Buss2006}. Our normalization theorem for the $\W(\Maj)$-hierarchy is inspired from those.} 
Following this approach, we start by showing loose normalization theorem for the $\W(\Maj)$-hierarchy in Lemma~\ref{lemma:maj3t+2}; namely, we prove that
the problem~$WCS(M_{3t+2,
3t+3})$ is~$\W(\Maj)[t]$-hard. The main difficulty here is to reduce the
depth~$d$ of the majority circuits, for any fixed~$d\in \mathbb{N}$, to a depth
of at most~$3t+3$.  We then show in Lemma~\ref{lemma:circuitMajToRelu} that
rMLPs can simulate majority circuits, without increasing the depth of the
circuit. In Theorem~\ref{thm:layers-hardness} we use this construction to show
an fpt-reduction from~$WCS(M_{3t+2, 3t+3})$ to~$(3t+3)$-MCR. This is enough to
conclude hardness for~$\W(\Maj)[t]$.

\textbf{Membership.} 
We prove membership in Section~\ref{subsec:p-membership}.
Presented in Theorem~\ref{thm:layers-membership}, the proof consists of 4 steps. We first show in Lemma~\ref{lemma:equivReluThreshold} how to transform a given rMLP~$\M$ that into an MLP~$\M'$ that uses only step activation functions and that has the same number of layers.  Then, as a second step, we build an MLP~$\M''$, with~$3t+4$ layers and again using only the step activation function, such that~$\M''$ has a satisfying assignment of weight~$k$ if and only if~$(\M, \vx, k)$ is a positive instance of the~$t$-MCR problem. 
The third step is to use a result of circuit complexity \cite{Goldmann1998} stating that circuits with weighted thresholds gates (which are equivalent to biased step functions), can be transformed into circuits using only majority gates, increasing the depth by no more than 1. 
This yields a circuit~$C_{\M''}$ with~$3t+5$ layers. However, the circuit~$C_{\M''}$, resulting from the construction of Goldmann et al. \cite{Goldmann1998}, has both positive variables and negated variables as inputs, as their model needs to be able to represent non-monotone functions. For the fourth and last step, we build a circuit~$C^*_{\M''}$ based on~$C_{\M''}$, that fits the description of majority circuits as defined by~\cite{Fellows, Fellows2007CombinatorialCA} (i.e., the one that we use).
 This circuit~$C^*_{\M''}$ has weft~$3t+7$, and we prove that~$(C^*_{\M''}, k+1)$ is a positive instance of the Weighted Circuit Satisfiability problem that characterizes the class~$\W(\Maj)[t]$ if and only if~$(\M, \vx, k)$ is a positive instance of the~$(3t+3)$-MCR problem. The whole construction being an fpt-reduction, this will be enough to conclude membership in~$\W(\textsc{Maj})[3t+7]$.
 
Observe that (r)MLPs can be interpreted as well as rooted directed acyclic graphs, with weighted edges and where each node is associated a layer according to its (unweighted) distance from the root. Every node in a certain layer~$\ell$ is connected to every node in layers~$\ell-1$ and~$\ell+1$. We will sometimes use this equivalent interpretation, which turns out to be more handy for some of the proofs in this section.

\subsection{Hardness}
\label{subsec:p-hardness}
As explained in the proof sketch, we start by establishing a normalization theorem for the $\W(\Maj)$-hierarchy.
\begin{lemma}
	The problem~$WCS(M_{3t+2,3t+3})$ is~$\W(\Maj)[t]$-hard.
	\label{lemma:maj3t+2}  
\end{lemma}
\begin{proof}
A significant part of this proof is based on techniques due to Fellows et al.~\cite{Fellows} and to Buss et al.~\cite{Buss2006}.
	Let~$C$ be an arbitrary majority circuit of weft at most~$t$ and depth at most~$d \geq t$ for some constant~$d$, and let~$k$ be the parameter of the input instance. We define a \emph{small sub-circuit} as a maximally connected sub-circuit comprising only small gates. Now, consider a path~$\pi$ from an arbitrary input node of~$C$ to its output gate. We claim that~$\pi$ intersects at most~$t+1$ small sub-circuits. Indeed, there must be at least one large gate separating every pair of small sub-circuits intersected by~$\pi$, as otherwise the maximality assumption would be broken. But in~$\pi$, as in any path, there are at most~$t$ large gates, because of the weft restriction, from where we conclude the claim. Now, for each small sub-circuit~$S$, consider the set~$I_S$ of its inputs (that may be either large gates or input nodes of~$C$). As small gates have fan-in at most~$3$, and the depth of each small sub-circuit is at most~$d$, we have that~$|I_S| \leq 3^d$. We can thus enumerate in constant time all the satisfying assignments of~$S$. We identify each assignment with the set of variables to which it assigns the value~$1$. We keep a set~$\Gamma~$ with the satisfying assignments among~$I_S$ that are minimal with respect to~$\subseteq$. Then, because of the fact that majority circuits are monotone,~$S$ can be written in monotone DNF as
	\[
		S \equiv \bigvee_{\gamma \in \Gamma} \bigwedge_{x\in \gamma} x 
	\]

	Note that the size of~$\Gamma$ is trivially bounded by the constant~$2^{3^d}$. We then build a circuit~$C'$, based on~$C$, by following these steps:
	\begin{enumerate}
		\item Add~${3^d}(k+1)$ extra input nodes. We distinguish the first, that we denote as~$u$, from the~$3^d(k+1)-1$ remaining, that we refer to by~$N$.
		\item Add a new output gate that is a binary majority between the old output gate and the node~$u$.
		\item Replace every small sub-circuit~$S$ by its equivalent monotone DNF formula, consisting of one large~$\gor$-gate and many large~$\gand$-gates.
		\item Relabel every large~$\gor$-gate, of fan-in~$\ell \leq 2^{3^d}$ created in the previous step to be a majority gate with the same inputs, but to which one wires as well~$\ell$ parallel edges from the input node~$u$.
		\item Relabel every large~$\gand$-gate~$g$, of fan-in~$\ell \leq 3^d$, to be a majority gate. If~$g$ had edges from gates~$g_1, \ldots, g_\ell$, then replace each edge coming from a~$g_i$ by~$k+1$ parallel edges, and finally, wire~$\ell (k+1) - 1$ nodes in~$N$ to~$g$.
	\end{enumerate}
	
      An illustration of the transformation ins presented in Figure~\ref{fig:normalization}. We now check that~$C'$ is a (majority) circuit in~$M_{3t+2,3t+3}$.  To
bound the depth and weft of~$C'$ we need to account for all the sub-circuits of
depth~$2$ that we introduced in steps~$3$--$5$ to replace each small
sub-circuit of~$C$. Note that two small sub-circuits that were parallel in~$C$
(meaning no input-output path could intersect both) have corresponding
sub-circuits that are parallel in~$C'$.  Consider now an arbitrary path~$\pi$
from a variable to the root of~$C$, and let~$\pi'$ be the corresponding path
in~$C'$ (that goes to the new root of~$C'$). The path~$\pi$ contains one
variable gate, at most~$t$ large gates, and intersects at most~$t+1$ small
sub-circuits.  The corresponding path~$\pi'$ in~$C'$ still contains the
variable gate, the (at most~$t$) large gates that were in~$\pi$, and for each
of the~at most~$t+1$ small-subcircuits that~$\pi$ intersected,~$\pi'$ now contains
exactly~$2$ large gate (and~$\pi'$ also contains the new output gate of~$C'$).
Therefore, the length of~$\pi'$ is at most~$1+t+2(t+1)+1-1=3t+3$, and it
contains at most~$t+2(t+1)=3t+2$ large gates.  Since every path~$\pi'$ in~$C'$
from a variable to the root of~$C'$ corresponds to such a path~$\pi$ in~$C$, we obtain
that the depth of~$C'$ is at most~$3t+3$ and its weft is at most~$3t+2$.
Hence,~$C'$ is indeed a majority circuit in~$M_{3t+2,3t+3}$.

We now prove that ($\star$) there is a satisfying assignment of weight~$k+1$
for~$C'$ if and only if there is a satisfying assignment of weight~$k$ for~$C$,
which would conclude our fpt-reduction.  The proof for this claim is based on
how the constructions in step 4 and 5 actually simulate large~$\gor$-gates
and~$\gand$-gates, respectively.\footnote{Although this technique can already
be found in the work of Fellows et al.  \cite{Fellows}, we include it here for
completeness.}
We prove each direction in turn.
	
\paragraph*{Forward direction.} Let us assume that there exists a satisfying assignment of weight~$k+1$
for~$C'$. First, because input node~$u$ is directly connected to the output gate
through a binary majority, it must be assigned to~$1$ in order to satisfy~$C'$.
Let~$C''$ be the sub-circuit of~$C'$ formed by all the nodes
that descend from the old output-gate in~$C'$. Then~$C''$ needs to be
satisfied in order to satisfy~$C'$. Since~$u$ is not present in~$C''$, an
assignment of weight~$k+1$ that satisfies~$C'$ is made by assigning~$1$ 
to~$u$ and to exactly~$k$ other input gates. In order to prove the
claim, we will show that ($\dagger$) an assignment of weight~$k$ for the inputs
of~$C''$ satisfies~$C''$ if and only if its restriction to the inputs of~$C$
satisfies~$C$, assuming~$u$ is assigned to~$1$. As~$C''$ only differs from~$C$
because of the replacement of each small sub-circuit~$S$ by its equivalent DNF,
and the additional inputs in~$N$, we only need to prove that steps 4 and 5
actually compute large~$\gor$ and~$\gand$ gates. Consider a gate~$g$ introduced
in step 4, having edges from gates~$g_1, \ldots, g_\ell$ and~$\ell$ edges from
node~$u$. Therefore,~$g$ has fan-in~$2\ell$, and as~$u$ always contributes with
a value of~$\ell$ to~$g$, we have that~$g$ is satisfied exactly when at least
one of the gates~$g_1, \ldots, g_\ell$ is satisfied. Consider now a
gate~$g$ introduced in step~$5$. By construction,~$g$ has fan-in equal
to~$2\ell(k+1) - 1$, from which we deduce that if all gates~$g_1, \ldots,
g_\ell$ are satisfied, then~$g$ is indeed satisfied in~$C''$. On the other hand, if an
assignment of weight~$k$ does not satisfy every gate~$g_i$, then~$g$ receives
at most~$(\ell-1)(k+1)~$ units from the gates~$g_i$, and as the assignment has
weight~$k$, it receives at most~$k$ from the nodes in~$N$. Thus,~$g$ receives
at most~$(k+1)\ell -1$ units, which is less than half of its fan-in, and
thus,~$g$ is not satisfied. 
Thus, we have proved ($\dagger$).
However, notice that the restriction of the assignment might have a weight of strictly less than~$k$ in~$C$.
But it is clear that, since the circuit is monotone, we can increase the weight by setting some variables of~$C$ to~$1$, until the weight becomes equal to~$k$.
This proves the forward direction.

\paragraph*{Backward direction.}
Let us now assume an assignment of
weight~$k$ for~$C$. We then we extend such an assignment to~$C'$ by assigning~$0$
to the inputs in~$N$ and~$1$ to~$u$. Thanks to ($\dagger$), this
is a satisfying assignment of weight~$k+1$ for~$C'$, which proves the backward direction of~($\star$) and thus concludes the proof of Lemma~\ref{lemma:maj3t+2}.
\end{proof}

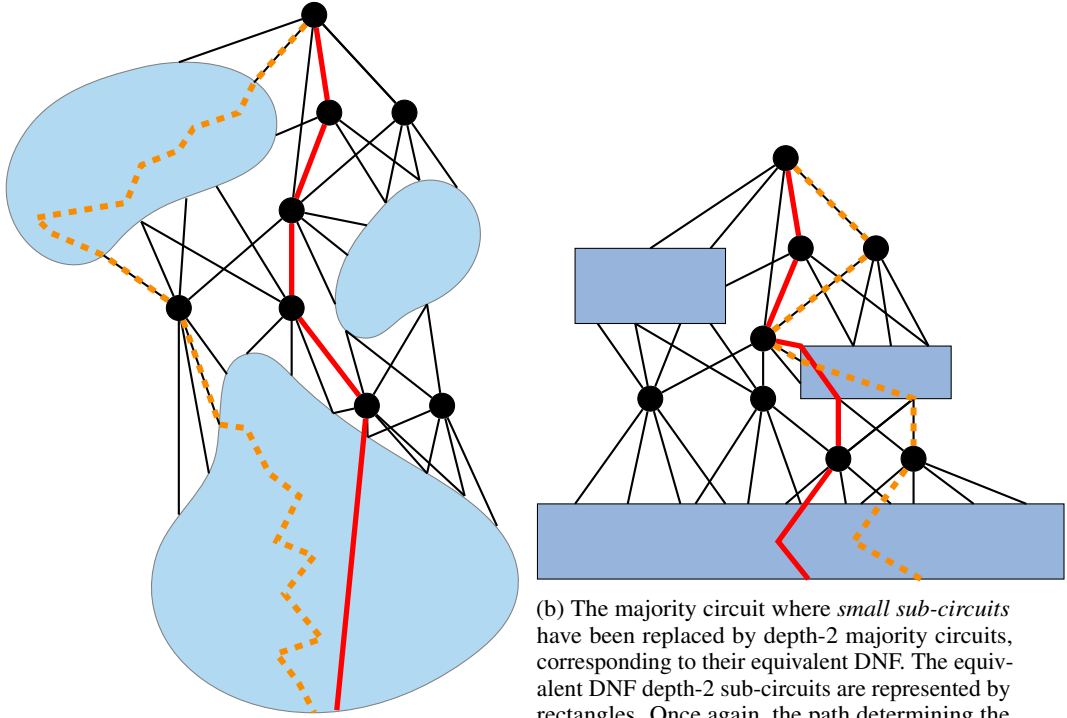
\begin{figure}
	\begin{subfigure}{0.45\textwidth}
\centering
\begin{tikzpicture}
	
	\def\vsp{1.3}
	\begin{scope}[every node/.style={draw, circle, fill=black}]
	
	
\definecolor{blobColor}{RGB}{100,180,230}

\path[draw, opacity=0.5, fill=blobColor, use Hobby shortcut,closed=true]
(-0.3, 0.6) .. (3.8, -0.8) .. (4.5, 0.1)
.. (3.3, 1.7) .. (1.5, 3.0) .. (1.0, 3.3) ..
(.5, 2);

	\node at (2.5,2*\vsp) (l22) {};
	\node at (3.5,2*\vsp) (l23) {};
	
	\node at (0,3*\vsp) (l31) {};
	\node at (1.5,3*\vsp) (l32) {};
	
	\node at (1.5,4*\vsp) (l41) {};
	
	\path[draw, opacity=0.5, fill=blobColor, use Hobby shortcut,closed=true]
(3,3.8) .. (3.4,4) .. (4,5) .. (3.3,5.6) .. (2.5,5) .. (2.5,3.5);

	\node at (2,5*\vsp) (l101) {};
	\node at (3,5*\vsp) (l102) {};
	
	\path[draw, opacity=0.5, fill=blobColor, use Hobby shortcut,closed=true]
(-0.8, 4.8) .. (-0.1,5.3) .. (1.3,6.3) .. (-0.5,7.1) .. (-2,6.3) .. (-1,4.6);
	
	\node at (1.8,6*\vsp) (l111) {};

	\end{scope}
	
	\begin{scope}[every path/.style={-, thick}]
	
	\path (2.51, 2.18) edge node {} (l22);
	\path (3.3, 1.7) edge node {} (l22);
	\path (3.8, 1.4) edge node {} (l22);
	\path (2.05, 2.5) edge node {} (l22);
	
	\path (2.51, 2.18) edge node {} (l23);
	\path (3.3, 1.7) edge node {} (l23);
	\path (3.8, 1.4) edge node {} (l23);
	\path (4.23, 1.0) edge node {} (l23);
	
	\path (1,6.9) edge node {} (l111);
	\path (0,7.17) edge node {} (l111);
	\path (l102) edge node {} (l111);
	\path (l101) edge node {} (l111);
	\path (l102) edge node {} (l111);
	\path (l102) edge node {} (l41);
	\path (l41) edge node {} (l111);

	\path (l41) edge node {} (2.08, 3.95);
	\path (l41) edge node {} (2.25, 4.6);
	\path (l41) edge node {} (2.5, 5);

	\path (l31) edge node {} (l41);
	\path (l32) edge node {} (l41);
	
	\path (l31) edge node {} (-.5,5.05);
	\path (l32) edge node {} (-.5,5.05);
	
	\path (l32) edge node {} (0.5,5.52);
	
	\path (l31) edge node {} (-1,4.6);

	\path (l31) edge node {} (0.1,5.36);
	
	\path (l31) edge node {} (0,1.15);
	\path (l31) edge node {} (.4,1.7);
	\path (l31) edge node {} (.54,2.35);
	\path (l31) edge node {} (0.64,3.02);
	
	\path (l32) edge node {} (0.9,3.29);
	\path (l32) edge node {} (2.05,2.5);
	\path (l32) edge node {} (1.5,3);

	\path (l32) edge node {} (l22);
	
	\path (2.22,3.6) edge node {} (l22);
	\path (2.22,3.6) edge node {} (l23);
	
	\path (3.3,3.95) edge node {} (l22);
	\path (3.3,3.95) edge node {} (l23);
	
	\path (2.75,5.3) edge node {} (l102);
	\path (3.2,5.6) edge node {} (l102);
	\path (3.7,5.5) edge node {} (l102);
	
	\path (1.29,6.2) edge node {} (l101);
	\path (2.75,5.3) edge node {} (l101);
	\path (3.2,5.6) edge node {} (l101);

	\path[draw, line width=0.7mm, red] (l111) -- (l101) -- (l41) -- (l32) -- (l22) -- (2.1, -1.45) ;

\definecolor{Forange}{RGB}{250,140,0}
	\path[draw, dashed, ultra thick, line width=0.8mm, color=Forange] (l111) -- (1, 6.9) -- (0.8, 6.5)
	-- (0.2, 6.3) -- (0, 6) --  (-0.5, 5.8) -- (-0.7, 5.3) -- (-1.9, 5.1) -- (-1.7, 4.9) -- (-1, 4.6) -- (l31) -- (.54, 2.35) -- (.9, 2.3)
	-- (1.2, 1.7) -- (1.6, 1.4) -- (1.3, .8) -- (1.8, .6) -- (1.4, .1) -- (1.9, -0.5) -- (1.3, -0.7) -- (1.6, -1.1) -- (1.8, -1.48);

	\end{scope}
\end{tikzpicture}
\caption{A majority circuit where \emph{small sub-circuits} are represented with blue blobs, and black nodes correspond to large majority gates. The path determining the \emph{weft} is colored red. The longest path, determining the \emph{depth} of the circuit, is drawn with a dashed orange line.}
\end{subfigure}
\hfill
\begin{subfigure}{0.45\textwidth}
\centering
\begin{tikzpicture}
	
	\def\vsp{0.8}
	\begin{scope}[every node/.style={draw, circle, fill=black}]
	
	
\definecolor{blobColor}{RGB}{150,180,220}

	\draw[fill=blobColor] (-1.5,1.6) rectangle ++(7,1);
	
	\draw[fill=blobColor] (2,4) rectangle ++(2,0.7);
	
	\draw[fill=blobColor] (-1,5) rectangle ++(2,1);

	\node at (2.5,4*\vsp) (l22) {};
	\node at (3.5,4*\vsp) (l23) {};
	
	\node at (0, 5*\vsp) (l31) {};
	\node at (1.5, 5*\vsp) (l32) {};
	
	\node at (1.5,6*\vsp) (l41) {};
	

	\node at (2, 7.5*\vsp) (l101) {};
	\node at (3, 7.5*\vsp) (l102) {};
	
	
	\node at (1.8, 9*\vsp) (l111) {};

	\end{scope}
	
	\begin{scope}[every path/.style={-, thick}]
	
	\path (l102) edge node {} (l111);
	\path (l101) edge node {} (l111);
	\path (l102) edge node {} (l41);
	\path (l41) edge node {} (l111);

	\path (0.8, 6) edge node {} (l111);
	\path (0,6) edge node {} (l111);
	
	\path (1, 5.5) edge node {} (l101);
	
	\path (2.7, 4.7) edge node {} (l101);
	\path (3.7, 4.7) edge node {} (l101);
	
	\path (2.7, 4.7) edge node {} (l102);
	\path (3.2, 4.7) edge node {} (l102);
	\path (3.7, 4.7) edge node {} (l102);

	\path (-1, 2.6) edge node {} (l31);
	\path (-0.3, 2.6) edge node {} (l31);
	\path (0.4, 2.6) edge node {} (l31);
	\path (1, 2.6) edge node {} (l31);
	
	\path (0.6, 2.6) edge node {} (l32);
	\path (1.3, 2.6) edge node {} (l32);
	\path (2, 2.6) edge node {} (l32);
	
	\path (1.8, 2.6) edge node {} (l22);
	\path (2.6, 2.6) edge node {} (l22);
	\path (3.2, 2.6) edge node {} (l22);
	
	\path (2.8, 2.6) edge node {} (l23);
	\path (3.5, 2.6) edge node {} (l23);
	\path (4.3, 2.6) edge node {} (l23);
	\path (5, 2.6) edge node {} (l23);

	\path (l22) edge node {} (2.5,4);
	\path (l23) edge node {} (3.5,4);
	\path (l23) edge node {} (2.5,4);
	\path (l22) edge node {} (3.5,4);
	
	\path (l22) edge node {} (3.5,4);
	
	\path (l41) edge node {} (2,4.2);
	\path (l41) edge node {} (2,4.5);
	\path (l41) edge node {} (2,4.7);
		
	\path (l31) edge node {} (l41);
	\path (l32) edge node {} (l41);
	
	\path (l31) edge node {} (-.2,5);
	\path (l31) edge node {} (-.7,5);
	\path (l31) edge node {} (0.4,5);
	
	\path (l32) edge node {} (-.2,5);
	
	\path (l32) edge node {} (0.8,5);
	\path (l32) edge node {} (l22);
	
\path[draw, line width=0.7mm, red] (l111) -- (l101) -- (l41) -- (2, 4.7) -- (2.5, 4) -- (l22) -- (1.7, 2.1) -- (2.1, 1.6);

\definecolor{Forange}{RGB}{250,140,0}
\path[draw, dashed, ultra thick, line width=0.8mm, color=Forange] (l111) -- (l102) -- (l41) -- (2, 4.5) -- (3.5, 4) -- (l23) -- (2.7, 2.1) -- (3.6, 1.6);

	\end{scope}
\end{tikzpicture}
\caption{The majority circuit where \emph{small sub-circuits} have been replaced by depth-2 majority circuits, corresponding to their equivalent DNF. The equivalent DNF depth-2 sub-circuits are represented by rectangles. Once again, the path determining the \emph{weft} is colored red. The longest path, determining the \emph{depth} of the circuit, is drawn with a dashed orange line.}
\end{subfigure}
	\caption{Illustration of the Normalization Lemma (\ref{lemma:maj3t+2}). In a nutshell, by paying a controlled increase in weft, the depth of the circuit can be substantially reduced.}
	\label{fig:normalization}
\end{figure}

Then, we show that rMLPs can simulate majority circuits, without increasing the
depth of the circuit.

\begin{lemma}
	Given a circuit~$C$ containing only majority gates, we can build in polynomial time an rMLP that is equivalent to~$C$ (as a Boolean function)
and whose number of layers is equal to the depth of~$C$.
	\label{lemma:circuitMajToRelu}
\end{lemma}

\begin{proof}
    First, note that we can assume that circuit~$C$ does not contain parallel edges by replacing each gate $g$ having $p$ edges to a gate $g'$ by $p$ copies $g_1, \ldots, g_p$ with single edges to $g'$.
    We then build a layerized circuit (remember the definition of a layerized
circuit from Appendix~\ref{sec:simulation})~$C'$ from~$C$, by applying the same
construction that we used in Lemma~\ref{lem:circuits-to-MLPs} to layerize a
circuit, but using unary majority gates as identity gates
instead.  Note that the depth of~$C'$ is the same as that of~$C$.
	
    Next, we show how each non-output majority gate can be simulated by using
two~$\relu$-gates (again, remember the definition of a relu gate from
Appendix~\ref{sec:simulation}). First, note that ($\dagger$) for any
non-negative integers~$x,n\in \mathbb{N}$, the function 
	\[
	f_n(x) \coloneqq \relu \left(x-\lfloor \frac{n}{2} \rfloor \right) - \relu\left(x - \lfloor \frac{n}{2} \rfloor - 1 \right)
	\]
	 is equal to 
	 \[
	 \Maj_n(x) = \begin{cases}
		1 & \text { if } x > \frac{n}{2}\\
		0 & \text{otherwise}
	\end{cases}.
	\]
    We will use ($\dagger$) to transform the majority circuit~$C'$ into a
circuit~$C''$ that has only relu gates for the non-output gates, and that is equivalent to $C'$ in a sense
that we will explain next.  For every non-output majority gate~$g$ of~$C'$, we create two
relu gates~$g'_1,g'_2$ of~$C''$. The idea is that~($\star$) for any valuation
of the input gates (we identify the input gates of~$C'$ with those of~$C''$),
the Boolean value of any non-output gate~$g$ in~$C'$ will be equal to the (not necessarily
Boolean) value of gate~$g'_1$ (in~$C''$) minus the value of the gate~$g'_1$
(in~$C''$).  We now explain what the biases of these new gates~$g'_1,g'_2$ 
for every majority gate~$g$ of~$C'$ are.  Letting~$n$ be the in-degree of a
majority gate~$g$ in~$C'$, the bias of~$g'_1$ is~$-\lfloor \frac{n}{2}
\rfloor$, and that of~$g'_2$ is~$- \lfloor \frac{n}{2} \rfloor - 1$.  Next, we
explain what the weights of these new gates~$g'_1,g'_2$ are and how we connect
them to the other relu gates. We do this by a bottom-up induction on~$C'$, that
is, on the level of the gates of~$C'$ (since~$C'$ is layerized), and we will at
the same time show that~($\star$) is satisfied. To connect the
gates~$g'_1,g'_2$ to the gates of the preceding layer, we differentiate two
cases:
\begin{description}
    \item[Base case.] The inputs of the gate~$g$ are variable gates; in other words, the level of~$g$ in~$C'$ is~$1$ (remember that variable gates have level~$0$). We then set these variable gates to be an input of both~$g'_2$ and~$g'_2$, and set all the weights to~$1$. It is clear that~($\star$) is satisfied for the gates~$g,g'_1,g'_2$, thanks to ($\dagger$).
    \item[Inductive case.]  The inputs of the gate~$g$ are other majority gates; in other words, the level of~$g$ in~$C'$ is~$>1$. Then, let~$^1g,\ldots, ^mg$ be the inputs\footnote{Please excuse us for using left superscripts.} (majority gates) of the gate~$g$ in~$C'$, and consider their associated pairs of relu gates~$(^1g'_1,^1g'_2),\ldots,(^mg'_1,^mg'_2)$ in~$C''$. We then set all the gates~$^1g'_1,\ldots, ^mg_1$ to be input gates of both gates~$g'_1$ and~$g'_2$, with a weight of~$1$, and 
set all the gates~$^1g_2,\ldots, ^mg_2$ to be input gates of both gates~$g'_1$ and~$g'_2$, with a weight of~$-1$. By induction hypothesis, and using again ($\dagger$), it is clear that ($\star$) is satisfied.
\end{description}

Finally, based on the output gate $r$ of $C'$, we create a step gate $r'$ in~$C''$ in the following way. Let $^1g,\ldots, ^mg$ be the inputs of $r$, and $(^1g'_1,^1g'_2),\ldots,(^mg'_1,^mg'_2)$ their associated pairs in~$C''$. Then wire each gate $^ig'_1$ to $r'$ with weight $1$, and also wire each gate $^ig'_2$ to $r'$ with weight $-1$. Let $-\lfloor \frac{n}{2} \rfloor - 1$ be the bias of $r'$. 

We have constructed a circuit~$C''$ whose output gate is a step gate, and all other gates are relu gates. 
Consider now a valuation $\vx$ of the input gates of $C'$, which we identify as well as a valuation $\vx'$ of the input gates of $C''$. We claim that $C'(\vx) = 1$ if and only if $C''(\vx') = 1$. 
But this simply comes from the fact that for~$x,n\in \mathbb{N}$, we have~$x> \frac{n}{2} \iff x \geq \lfloor \frac{n}{2} \rfloor +1$, and from the fact that~($\star$) is satisfied for the input gates of~$r$ and of~$r'$.

 The last thing that we have to do is to transform the
circuit~$C''$, that uses only relu gates except for its output step gate, into
a valid MLP.  This can be done easily as in the proof of
Lemma~\ref{lem:circuits-to-MLPs} by adding dummy connections with weights zero,
because~$C''$ is layerized.  The resulting MLP~$\M_C$ is then equivalent
to~$C$, it is clearly an rMLP, its number of layers is exactly the depth of~$C$, and, since we have constructed it in polynomial
time, this concludes the proof. 
\end{proof}

Finally, we use this construction to show
an fpt-reduction from~$WCS(M_{3t+2, 3t+3})$ to~$(3t+3)$-MCR. This is enough to
conclude hardness for~$\W(\Maj)[t]$, thanks to Lemma~\ref{lemma:maj3t+2}.
 
\begin{theorem}
There is an fpt-reduction from the problem~$WCS(M_{3t+2, 3t+3})$ to the~$(3t+3)$-MCR problem.
	\label{thm:layers-hardness}
\end{theorem}

\begin{proof}
	We will in fact show an fpt-reduction from~$WCS(M_{t, t})$ to~$t$-MCR, which gives the claim when applied to~$3t+3$, noting of course that $WCS(M_{3t+3, 3t+3})$ is trivially at least as hard as $WCS(M_{3t+2, 3t+3})$.
	Let~$(C,k)$ be an instance of~$WCS(M_{t,t})$. 
	We first build an MLP~$\M_C$ equivalent to~$C$ (as Boolean functions) 
    by using Lemma~\ref{lemma:circuitMajToRelu}. The MLP~$\M_C$ has~$t$ layers. Then, we build an MLP~$\M'_C$, that is based on~$\M_C$, by following the steps described below:

	\begin{enumerate}
		\item Initialize~$\M'_C$ to be an exact copy of~$\M_C$.
		\item Add an extra input, that we call~$v_1$, to~$\M'_C$. This means that if~$\M_C$ had dimension~$n$, then~$\M'_C$ has dimension~$n+1$.
		\item Create nodes~$v_2, \ldots, v_{t}$, all having a bias of~$0$, and for each~$1 \leq i < t$, connect node~$v_i$ to node~$v_{i+1}$ with an edge of weight~$1$.
		\item Let~$r$ be the root of~$\M'_C$, and let~$m$ be its fan-in. We connect node~$v_{t}$ to~$r$ with an edge of weight~$m$. Moreover, if the bias of~$r$ in~$\M_C$ was~$b$, we set it to be~$b-m$ in~$\M'_C$.	
        \item Observe that $\M'_C$ is layerized. To make it a valid MLP (where all the neurons of a layer are connected to all the neurons of the adjacent layers), we do as in the proof of Lemma~\ref{lem:circuits-to-MLPs} by adding dummy null weights.
	\end{enumerate} 
	
	It is clear that the construction of~$\M'_C$ takes polynomial time, and that its number of layers is again~$t$. We now prove a claim describing the behavior of~$\M'_C$.
	
	\begin{claim}
	For any instance~$\vx'$ of~$\M'_C$, expressed as the concatenation of a feature~$\vx'_1$ (for the extra input node~$v_1$) and an instance~$\vx$ of~$\M_C$, we have that~$\vx'$ is a positive instance of~$\M'_C$ if and only if~$\vx'_1 = 1$ and~$\vx$ is a positive instance of~$\M_C$
	\end{claim}
	\begin{proof}
	Consider that, by construction, an instance~$\vx'$ is positive for~$\M'_C$ if and only if 
	\[
		\sum_{i=1}^{n+1} \vh'^{(t-1)}_i \mW'^{(t)}_i = m \vh'^{(t-1)}_1 + \sum_{i=2}^{n+1} \vh'^{(t-1)}_i \mW'^{(t)}_i \geq -b+m 	
	\]
	
	But by construction~$\vh'^{(t-1)}_1 = \vx'_1$, and~$\sum_{i=2}^{m+1} \vh'^{(t-1)}_i \mW'^{(t)}_i = \sum_{i=1}^{m} \vh^{(t-1)}_i \mW^{(t)}_i$. This means that~$\vx'$ is a positive instance of~$\M'_C$ if and only if
	
	\[
		m\vx'_1 + \sum_{i=1}^{m} \vh^{(t-1)}_i \mW^{(t)}_i \geq -b + m
	\]
	
	Note that if~$\vx'_1 = 1$ and~$\vx$ is a positive instance of~$\M_C$, this inequality is achieved, making~$\vx'$ a positive instance. For the other direction, it is
clear that it holds if~$\vx'_1=1$. We show that in fact~$\vx'_1 = 0$ is not possible. 
	Indeed, by the construction of~$\M_C$, we have that~$0 \leq \sum_{i=1}^{m} \vh^{(t-1)}_i \mW^{(t)}_i \leq m$, and also that~$-b \geq 1$, which makes the inequality unfeasible.  
	
	This concludes the proof of the claim.
	\end{proof}
	
	This claim has two important consequences:
	\begin{enumerate}
		\item As satisfying assignments of~$C$ correspond to positive instance of~$\M_C$, we have that there is a satisfying assignment of weight exactly~$k$ for~$C$ if and only if there is a positive instance of weight exactly~$k+1$ for~$\M'_C$.
		\item The instance~$0^{n+1}$ is negative for~$\M'_C$
	\end{enumerate} 
	
	This consequences will allow us to finish the reduction. Consider the instance~$(\M'_C, 0^{n+1}, k+1)$ of~$t$-MCR. We claim that this is a positive instance for the problem if and only if~$(C, k)$ is a positive instance of~$WCS(M_t)$.

	For the forward direction, consider~$(\M'_C, 0^{n+1}, k+1)$ to be a positive instance of~$t$-MCR. This means there is an instance~$\vx^*$ that has the opposite classification as~$0^{n+1}$ under~$\M'_C$, and differs from it in at most~$k+1$ features. By the second consequence of the claim,~$\vx^*$ must be a positive instance. Also, differing in at most~$k+1$ features from~$0^{n+1}$ means that~$\vx^*$ has weight at most~$k+1$. But as majority gates are monotone connectives, majority circuits are monotones as well, so the existence of a positive instance~$\vx^*$ of weight at most~$k+1$ implies the existence of a positive instance~$\vx'^*$ of weight exactly~$k+1$.
	Therefore, by the first consequence of the claim, there is a satisfying assignment of weight exactly~$k$ for~$C$, which implies~$(C, k)$ is a positive instance of~$WCS(M_{t,t})$
	
	For the backward direction, consider~$(C, k)$ to be a positive instance of~$WCS(M_{t,t})$. This means, by the first consequence of the claim, that there is a positive instance~$\vx^*$ of weight exactly~$k+1$ 	for~$\M'_C$. But based on the second consequence of the claim,~$0^{n+1}$ is a negative instance for~$\M'_C$. As~$\vx^*$ differs from~$0^{n+1}$ in no more than~$k+1$ features, and they have opposite classifications, we have that~$(\M'_C, 0^{n+1}, k+1)$ is a positive instance of~$t$-MCR.
	
	As the whole construction takes polynomial time, and the reduction changes the parameter in a computable way, from~$k$ to~$k+1$, it is an fpt-reduction. This concludes the proof.
	\end{proof}

\subsection{Membership}
\label{subsec:p-membership}
In this section we prove membership in~$\W(\Maj)[3t+7]$. This will be enough to prove:

\begin{theorem}
There is an fpt-reduction from~$t$-MCR to~$WCS(M_{t+4, t+4})$, implying~$(3t+3)$-MCR belongs to~$\W(\Maj)[3t+7]$.	
\label{thm:layers-membership}
\end{theorem}

As explained in the proof sketch, we first show how to transform a given
rMLP~$\M$ that into an MLP~$\M'$ that uses only step activation functions and
that has the same number of layers.  More formally, we prove that rMLPs using
only step activation functions are powerful enough to simulate MLPs that use
\relu\ activation functions in the internal layers (and a step function for the
output neuron). The construction is polynomial in
the width (maximal number of neurons in a layer) of the given \relu-rMLP, but
exponential on its depth (number of layers).  We show:

\begin{lemma}
    Given an rMLP $\mathcal{M}$ with \relu\ activation functions, there is an
equivalent MLP $\mathcal{M}'$ that uses only step activation functions and has
the same number of layers.  Moreover, if the number of layers of $\mathcal{M}$
is bounded by a constant, then $\mathcal{M}'$ can be computed in polynomial
time.
\label{lemma:equivReluThreshold}
\end{lemma}

\begin{proof} 
Let $(\mW^{(1)},\ldots,\mW^{(\ell)})$, $(\vb^{(1)},\ldots,\vb^{(\ell)})$ and
$(f^{(1)},\ldots,f^{(\ell)})$ be the sequences of weights, biases, and activation
functions of the rMLP~$\M$. Note that~$f^{(i)}$ for~$1\leq i \leq \ell-1$
is~$\relu$ and that~$f^{(\ell)}$ is the step activation function.  
The first step of the proof is to transform every weight and bias into an integer.
To this end, let~$L\in \mathbb{N}$,~$L>0$ be the lowest common denominator of all the weights and biases, and let
$\mathcal{M}'$ be the MLP that is exactly equal to~$\mathcal{M}$ except that
all the weights have been multiplied by~$L$, and all the biases of layer~$i$
have been multiplied by~$L^i$.  Observe that $\mathcal{M}'$ has only integer
weights and biases.  When~$w$ (resp.,~$b$) is a weight (resp., bias) of~$\M$,
we write~$w'$ (resp.,~$b'$) the corresponding value in~$\M'$.  We claim
that~$\mathcal{M}$ and~$\mathcal{M}'$ are equivalent, in the sense that for
every $\vx \in \{0, 1\}^n$, it holds that $\mathcal{M}(\vx) =
\mathcal{M}'(\vx)$. 
Indeed, for~$0\leq i \leq \ell$, let~$\vh^{(i)}$ and~$\vh'^{(i)}$ be the vectors
of values for the layers of~$\M$ and~$\M'$, respectively, as defined by
Equation~\ref{eq:mlp}.  We will show that ($\star$) for all~$1\leq i \leq \ell-1$ 
we have~$\vh'^{(i)} = L^i\times \vh^{(i)}$. The base case of~$i=0$ (i.e., the
inputs) is trivially true. For the inductive case, assume that ($\star$) holds
up to~$i$ and let us show that it holds for~$i+1$. We have:

		\begin{equation*}
		\begin{split}
			\vh'^{(i+1)} & =  \relu(\vh'^{(i)}\mW'^{(i+1)} + \vb'^{(i+1)}) \\
			& =  \relu(L \times \vh'^{(i)}\mW^{(i+1)} + L^{i+1}\times \vb^{(i+1)}) \text{  by the definition of~$\M'$}\\
			& =  \relu(L^{i+1} \times \vh^{(i)}\mW^{(i+1)} + L^{i+1}\times \vb^{(i+1)}) \text{  by inductive hypothesis} \\
			& = L^{i+1} \times \relu(\vh^{(i)}\mW^{(i+1)} + \vb^{(i+1)}) \text{  by the linearity of $\relu$} \\
			& = L^{i+1} \times \vh^{(i+1)}, \\
		\end{split}
	\end{equation*}
and ($\star$) is proven. Since the step function (used for the output neuron) satisfies $\step(cx)=c\step(x)$ for~$c>0$, we indeed have that $\mathcal{M}(\vx) = \mathcal{M}'(\vx)$.

We now show how to build a model~$\M''$ that uses only step activation
functions and that is equivalent to~$\M'$. The first step is to prove an upper
bound for the values in~$\vh'$. We start by bounding the values in~$\vh$.
Let~$D$ be width of~$\M$, that is, the maximal dimension of a layer of~$\M$,
and let~$C$ be the maximal absolute value of a weight or bias in~$\M$; note
that the value of~$C$ is asymptotically bounded by $|\M|^{O(1)}$  because $\M$ is an rMLP.  For every instance~$\vx$, we have that 
\[0 \leq h_j^{(i)} = \relu \left(\sum_k h_k^{(i-1)} W^{(i)}_{k,j} + b^{(i)}_j \right) \leq D C \max_k{h_k^{(i-1)}} + C \leq (D+1)C \max(1,\max_k{h_k^{(i-1)}})\]
  Using this inequality, and the fact that
$\max_k{h_k^{(0)}} \leq 1$, we obtain inductively that $0 \leq
h_j^{(i)} \leq ((D+1)C)^{i}$. By ($\star$), this implies that $0 \leq
h'^{(i)}_j \leq ((D+1)CL)^{i}$.

As all values (weights, biases and the $\vh'$ vectors) in $\M'$ consist only of integers, and are all bounded by the integer $S \coloneqq ((D+1)CL)^{\ell}$, then each \relu\ in~$\M'$ with bias~$b$ becomes equivalent to the following function~$f^*$:
 
 \begin{equation}
 	 f^*(x+b) \coloneqq [x+b \ge 1] + [x+b \ge 2] + \ldots + [x+b \ge S]	
 	 \label{eq:stepSim}
 \end{equation}
 
 Where $[y \ge j] \coloneqq 1$ if $y \ge j$ and $0$ otherwise. Hence, in order to finish the proof, it is enough to show how activation functions of the form $f^*$ can be simulated with step activation functions. Namely, we show how to build $\M''$, that uses only step activation functions, from $\M'$, in such a way that both models are equivalent. In order to do so, we replace each $f^{(i)}, \mW'^{(i)}, \vb'^{(i)}$ for~$1\leq i \leq \ell$ in the following way. 
If $i = \ell$, then nothing needs to be done, as $f^{(\ell)}$ is already assumed to be a step activation function. When $1 \leq i < \ell$, we replace the weights, activations and biases in a way that is better described in terms of the underlying graph of the MLP.
 We split every internal node, with bias $b$ into $S$ copies, all of which will have the same incoming and outgoing edges as the original nodes, with the same weights. The $j$-th copy
will have a bias equal to $b-j$. We illustrated this step in Figure~\ref{fig:trick}. This construction is an exact simulation of the function $f^*$ defined in Equation~\ref{eq:stepSim}.
 
 The computationally expensive part of the algorithm is the replacement of each node in~$\M'$ by~$S$ nodes, which takes time at most~$S=((D+1)CL)^{\ell} \in O(|\M|^{\ell}(CL)^{\ell})$ per node and thus at most~$O(|\M|^{\ell+1}(CL)^{\ell})$ in total. 
Since~$\ell$ is a constant, and~$C$ is bounded by a polynomial on~$\M$, we only need to argue that~$L$ is bounded as well. Indeed, as~$\M$ is an rMLP, each weight and bias can be assumed to be represented as a fraction whose denominator is a power of~$10$ of value polynomial in the graph size~$N$ of~$\M$. But the lowest common multiple of a set of powers of~$10$ is exactly the largest power of~$10$ in the set. Therefore~$L \leq 10^p$, where~$p \in O(\log N)$, and thus~$L \in O(N^c) \subseteq O(|\M|^c)$ for some constant~$c$.
We conclude from this that the construction takes polynomial time.
 \end{proof}
 
  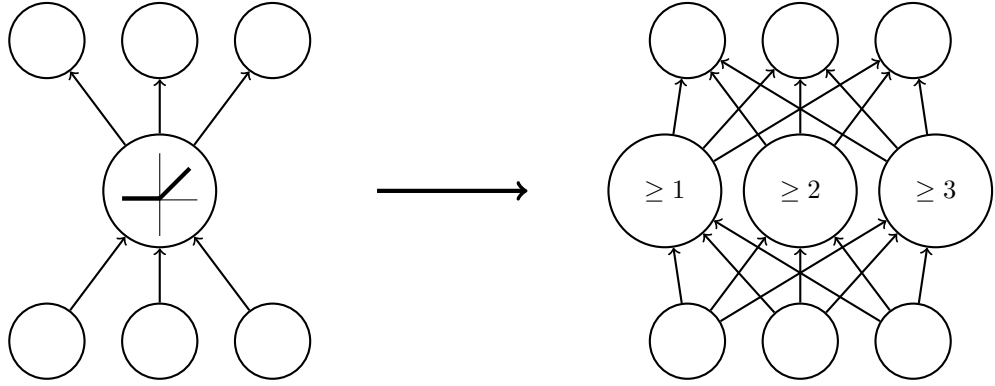
\begin{figure}
 	\begin{subfigure}{.4\textwidth}
 	\centering
 	\begin{tikzpicture}
 		\begin{scope}[every node/.style={circle,thick,draw, minimum size=1cm}]
			\node[minimum size=1.5cm] (1) at (0, 0) {};
			\node (2) at (-1.5, -2) {};
			\node (3) at (0, -2) {};
			\node (4) at (1.5, -2) {};
			\node (5) at (-1.5, 2) {};
			\node (6) at (0, 2) {};
			\node (7) at (1.5, 2) {};
					\end{scope}
		\draw[ultra thick] (-0.5,-0.1) -- (0, -0.1) -- (0.4, 0.3);
		\draw (0, -0.6) -> (0, 0.5);
		\draw (-0.5, -0.12) -> (0.5, -0.12);
		\begin{scope}[every path/.style={->, thick}]
			\path (1) edge node [left] {} (5);
			\path (1) edge node [left] {} (6);
			\path (1) edge node [right] {} (7);
			\path (2) edge node [left] {} (1);
			\path (3) edge node [left] {} (1);
			\path (4) edge node [left] {} (1);
		\end{scope}
		\end{tikzpicture}
 	\end{subfigure}
 	\begin{subfigure}{.2\textwidth}
 		\begin{tikzpicture}
 			\draw[ultra thick, ->] (0,0) --(2,0); 
 		\end{tikzpicture}
 	\end{subfigure}
	\begin{subfigure}{.4\textwidth}
	\centering
	\begin{tikzpicture}
		\begin{scope}[every node/.style={circle,thick,draw, minimum size=1cm}]
			\node[minimum size=1.5cm] (1a) at (-1.8, 0) {$\geq 1$};
			\node[minimum size=1.5cm] (1b) at (0, 0) {$ \geq 2$};
			\node[minimum size=1.5cm] (1c) at (1.8, 0) {$ \geq 3$};
			\node (2) at (-1.5, -2) {};
			\node (3) at (0, -2) {};
			\node (4) at (1.5, -2) {};
			\node (5) at (-1.5, 2) {};
			\node (6) at (0, 2) {};
			\node (7) at (1.5, 2) {};
					\end{scope}
		\begin{scope}[every path/.style={->, thick}]
			\path (1a) edge node [left] {} (5);
			\path (1a) edge node [left] {} (6);
			\path (1a) edge node [right] {} (7);
			\path (2) edge node [left] {} (1a);
			\path (3) edge node [left] {} (1a);
			\path (4) edge node [left] {} (1a);
			\path (1b) edge node [left] {} (5);
			\path (1b) edge node [left] {} (6);
			\path (1b) edge node [right] {} (7);
			\path (4) edge node [left] {} (1b);
			\path (2) edge node [left] {} (1b);
			\path (3) edge node [left] {} (1b);
			\path (1c) edge node [left] {} (5);
			\path (1c) edge node [left] {} (6);
			\path (1c) edge node [right] {} (7);
			\path (4) edge node [left] {} (1c);
			\path (2) edge node [left] {} (1c);
			\path (3) edge node [left] {} (1c);
		\end{scope}
	\end{tikzpicture}
	\end{subfigure}
	\caption{Illustration of the conversion from a \relu\ activation function to step activation functions, for $S = 3$. The weights are unchanged, and if the bias of the original neuron was~$b$ then the bias in the~$j$-th copy of that neuron becomes~$b-j$.}
    \label{fig:trick}
 \end{figure}

We are now ready to prove Theorem~\ref{thm:layers-membership}.

\begin{proof}[Proof of Theorem~\ref{thm:layers-membership}]
Let~$(\M, \vx, k)$ be an instance of~$t$-MCR. During this reduction we assume that~$n > 2k$, as otherwise the result can be achieved trivially; if~$n \leq 2k$ then trying all instances that differ by at most~$k$ from~$\vx$ takes only~$O(k^k)$, and thus we can solve the entire problem in fpt-time and return a constant-size instance of~$WCS(M_{t+2})$, completing the reduction.

We start by applying Lemma~\ref{lemma:equivReluThreshold} to build an equivalent MLP~$\M'$ that uses only step activation functions. As~$t$ is constant, this construction takes polynomial time, and its resulting MLP~$\M'$ has~$t$ layers as well. If~$\vx$ is a negative instance of~$\M'$ (and thus of~$\M$) we do nothing. This can trivially be checked in polynomial time, evaluating~$\vx$ in~$\M'$. But if~$\vx$ happens to be a positive instance of~$\M'$, then we change the definition of~$\M'$ negating its root perceptron\footnote{Let~$\mathcal{P} = (\vw, \vb)$ be the perceptron at the root of~$\M'$, which contains only integer values by construction. Then, the negation of~$\mathcal{P}$ is simply~$\bar{\mathcal{P}} = (-\vw, -\vb + 1)$, as~$-\vw \vx \geq -\vb + 1$ precisely when~$\vw \vx \leq \vb -1$, which occurs over the integers exactly when it is not true that~$\vw \vx \geq \vb$.}, and thus making~$\vx$ a negative instance. As a result, we can safely assume~$\vx$ to be a negative instance of~$\M'$.  We can also, in the same fashion that we assumed~$n > 2k$, discard the case where the instance~$0^n$ is a positive instance of~$\M'$ that differs by at most~$k$ from~$\vx$, as in such scenario we could also solve the problem in fpt-time. The same can be done for~$1^n$.

 We now build an MLP~$\M''$, that still uses only step activation functions, such that~$\M''$ has a positive instance of weight exactly~$k$ if and only if~$(\M, \vx, k)$ is a positive instance of~$t$-MCR.

Let~$\M''$ be a copy of~$\M'$ to which we add one extra layer at the bottom. For each~$1 \leq i \leq n$, we connect the~$i$-th input node of~$\M''$ to what was the~$i$-th input node of~$\M'$, but is now an internal node in~$\M''$. If~$\vx_i = 0$ then the node in~$\M''$ corresponding to the~$i$-th input node of~$\M'$ has a bias of~$1$, and the weight of the edge coming from the~$i$-th input node of~$\M''$ is also~$1$. On the other hand, if~$\vx_i = 1$, then the node in~$\M''$ corresponding to the~$i$-th input node of~$\M'$ has a bias of~$0$, and the weight of the connection added to it is~$-1$. After doing this, we add~$k-1$ more input nodes to~$\M''$, a new node~$p$ in the~$t$-th layer and a new root node~$r''$, that is placed in the layer~$t+1$. We connect~$r'$, the previous root node, to~$r''$ of~$\M'$ with weight~$1$, and all input nodes to node~$p$ with weights of~$1$. In case~$p$ is more than one layer above the new input nodes, we connect them through paths of identity gates, as shown in Lemma~\ref{lem:circuits-to-MLPs}. We set the bias of~$r''$ to~$-2$, and the bias of~$p$ to~$-k$. All non-input nodes added in the construction use step activation functions.	

We now prove a claim stating that~$\M''$ has exactly the intended behavior.

\begin{claim}
	The MLP~$\M''$ has a positive instance of weight exactly~$k$ if and only if~$(\M, \vx, k)$ is a positive instance of~$t$-MCR.
	\label{claim:M''}
\end{claim}

\begin{proof}
For the forward direction, assume~$\M''$ has a positive instance~$\vx'$ of weight exactly~$k$. As the root~$r''$ has a bias of~$-2$, and two incoming edges with weight~$1$, and given that the output of any node is bounded by~$1$, as only step activation functions are used, we conclude that both~$p$ and~$r'$, the children of~$r''$, must have a value of~$1$ on~$\vx'$. The fact that~$r'$ has a value of~$1$ on~$\vx'$ implies that~$\vx^s$, the restriction of~$\vx$ that considers only nodes that descend from~$r'$, must be a positive instance for the submodel~$\M^s$ induced by considering only nodes that descend from~$r'$.  But one can easily check that by construction, we have that~$\M^s(\vx^s) = \M'(\vx^s \oplus \vx)$, where~$\oplus$ represents the bitwise-xor. Thus,~$\vx^s \oplus \vx$ is a positive instance for~$\M$, and consequently for~$\M$. As~$\vx^s \oplus \vx$ differs from~$\vx$ by exactly the weight of~$\vx^s$, as~$0$ is the neutral element of~$\oplus$, and the weight of~$\vx^s$ is by definition no more than the weight of~$\vx'$, which is in turn no more than~$k$ by hypothesis, we conclude that~$(\M, \vx, k)$ is a positive instance of~$t$-MCR.

For the backward direction, assume there is a positive instance~$\vx'$  of~$\M$ that differs from~$\vx$ in at most~$k$ positions. This means that~$\vx''= \vx \oplus \vx'$ has weight at most~$k$. By the same argument used in the forward direction,~$\M^s(\vx'') = \M'(\vx'' \oplus \vx) = \M'(\vx')$, as~$\vx \oplus \vx' \oplus \vx = \vx \oplus \vx \oplus \vx' = \vx'$, because~$\oplus$ is both commutative and its own inverse. But the fact that~$\vx'$ is a positive instance of~$\M$ implies that it is also a positive instance for~$\M'$. As we are assuming~$\vx'| \neq 0^n$, we have that~$k- |\vx'| \leq k-1$. Thus, we can create an instance~$\vx''$ for~$\M''$ that is equal to~$\vx'$ on its corresponding features, and that sets~$k-|\vx'|$ arbitrary extra input nodes to~$1$, among those created in the construction of~$\M''$. As the instance~$\vx''$ has weight exactly~$k$, it satisfies the submodel descending from~$p$, and as~$\vx''$ its equal to~$\vx'$ on the submodel descending from~$r'$, and~$\vx'$ is a positive instance of~$\M'$, we have that this submodel must be satisfied as well. Both submodels being satisfied, the whole model~$\M''$ is satisfied, hence we conclude the proof.
\end{proof}

We thus have a model~$\M''$ with step activation functions, and~$t+2$ layers, such that if that model has a satisfying assignment of weight exactly~$k$, then~$(\M, \vx, k)$ is a positive instance of~$t$-MCR.

Note that step activation functions with bias are equivalent to weighted threshold gates. We then use a result by Goldmann and Karpinski~\cite[Corollary 12]{Goldmann1998} to build a circuit~$C_{\M''}$ that is equivalent (as Boolean functions) to~$\M''$ but uses only majority gates. The construction of Goldmann et al. can be carried in polynomial time, and guarantees that~$C_{\M''}$ will have at most~$t+3$ layers. 

There is however a caveat to surpass: although not explicitly stated in the work of Goldmann et al.~\cite{Goldmann1998}, their definition of majority circuit must assume that for representing a Boolean function from~$\{0, 1\}^n$ to~$\{0,1\}$, the circuit is granted access to~$2n$ input variables~$\vx_1,\ldots, \vx_n, \overline{\vx_1}, \ldots, \overline{\vx_n}$, as it is usual in the field, and described for example in the work of Allender \cite{Allender1989}. We thus assume that the circuit~$C_{\M''}$ resulting from the construction of Goldmann et al. has this structure, which does not match the required structure of the majority circuits defining the~$\W(\Maj)$-hierarchy as defined by Fellows et al \cite{Fellows, Fellows2007CombinatorialCA}. In order to solve this, we adapt a technique from Fellows et al. \cite[p. 17]{Fellows2007CombinatorialCA}. We build a circuit~$C^*_{\M''}$ that does fit the required structure. Let~$n$ be the dimension of~$\M''$ (which exceeds by~$k-1$ that of~$\M$). We now describe the steps one needs to apply to~$C_{\M''}$ in order to obtain~$C^*_{\M''}$.

\begin{enumerate}
	\item Add a new layer with~$n+1$ input nodes~$\vx'_1, \ldots, \vx'_{n+1}$, below what previously was the layer of~$2n$ input nodes~$\vx_1,\ldots, \vx_n, \overline{\vx_1}, \ldots, \overline{\vx_n}$.
	\item For every~$1 \leq i \leq n$, connect input node~$\vx'_i$ with its corresponding node~$\vx_i$ in the second layer, making~$\vx_i$ a unary majority, with the same outgoing edges it had as an input node. This enforces~$\vx_i = \vx'_i$.
	\item Create a new root~$r'$ for the circuit, and let~$r'$ be a binary majority between the input node~$\vx'_{n+1}$ and the previous root~$r$.
	
	\item Replace each previous input node~$\overline{\vx_i}$ by a majority gates~$m_i$ that has~$n+1-2k$ incoming edges from~$\vx'_{n+1}$, and one incoming edge from each~$\vx'_j$ with~$j \not \in  \{i, n+1\}	$. The outgoing edges are preserved.
\end{enumerate}

It is clear that the circuit~$C^*_{\M''}$ is a valid majority circuit in the sense defining the~$\W(\Maj)$-hierarchy. And it has~$2$ layers more than~$C_{\M''}$, yielding a total of~$t+5$ layers, where the last one has a small gate.  However, it is not evident what this new circuit does. We now prove a tight relationship between the circuit~$C^*_{\M''}$ and~$\M''$.

\begin{claim}
	The circuit~$C^*_{\M''}$ has a satisfying assignment of weight exactly~$k+1$ if and only if~$\M''$ has a positive instance of weight exactly~$k$.
	\label{claim:Cstar}
\end{claim}
 
 \begin{proof}
 
 	\textbf{Forward Direction.} Assume~$C^*_{\M''}$ has a satisfying assignment of weight~$k+1$. 
 	By step 3 of the construction, in order to satisfy~$C^*_{\M''}$, the assignment must set~$\vx'_{n+1}$ to~$1$. 
 
 	 As we assume that node~$\vx'_{n+1}$ is set to~$1$, the assignment must set to~$1$ exactly~$k$ input nodes among~$\vx'_1, \ldots, \vx'_n$ and thus the sum of inputs set to~$1$ of each majority gate~$m_i$ constructed in step 4, is exactly equal to 
 	\[
 	n+1-2k + \sum_{j \not \in  \{i, n+1\}}{\vx'_j} = n+1-2k + (k - \vx'_i) = n+1-k-\vx'_i
 	\]
 	and its fan-in is exactly equal to~$2n-2k$. Therefore~$m_i$ is activated when~$n+1-k-\vx'_i > n-k$, which happens precisely when~$\vx'_i = 0$. This way, each gate~$m_i$ corresponds to the negation of~$\vx'_i$. 

	This way, the subcircuit induced by considering only the nodes that descend from~$r'$ computes the same Boolean function that~$C_{\M''}$ computes, under the natural mapping of their variables. Therefore, a satisfying assignment of weight~$k+1$ for~$C^*_{\M''}$ implies the existence of a satisfying assignment for~$C_{\M''}$ that chooses exactly~$k$ positive variables, and thus a positive instance of weight~$k$ for~$\M''$.
	
	\textbf{Backward Direction.} Assume~$\M''$ has a positive instance of weight exactly~$k$. That implies that~$C_{\M''}$ has a satisfying assignment~$\sigma$ that sets at most~$k$ positive variables to~$1$. Let us consider the assignment~$\sigma'$ for~$C^*_{\M''}$ that sets to~$1$ the same variables that~$\sigma$ does, and additionally sets~$\vx_{n+1}$ to~$1$. 
 	The assignment~$\sigma'$ has weight exactly~$k+1$. By the same argument used in the forward direction, under assignment~$\sigma'$ the gates~$m_i$ behave like negations. Thus, the assignment~$\sigma'$ induces an assignment over the second layer of~$C^*_{\M''}$ that corresponds precisely to a satisfying assignment of~$C_{\M''}$, and thus makes the value of~$r$ equal to~$1$. As both~$r$ and~$\vx_{n+1}$ have value~$1$ under assignment~$\sigma'$,  it follows that the value of~$r'$, and thus of circuit~$C^*_{\M''}$, are~$1$ under~$\sigma'$ as well. This means that assignment~$\sigma'$, which by construction has weight~$k+1$, is a satisfying assignment for~$C^*_{\M''}$, and thus concludes the proof.
 	\end{proof}

 By combining Claim~\ref{claim:M''} and Claim~\ref{claim:Cstar}, and noting again that circuit~$C^*_{\M''}$ is a valid majority circuit, in the sense that defines the~$\W(\Maj)$-hierarchy, and has weft at most~$t+4$, we conclude the reduction of Theorem~\ref{thm:layers-membership}.
\end{proof}

%
%
\section{Proof of Proposition~\ref{prp:result-layers}}
\label{sec:proof-12}
Based on Proposition~\ref{prp:mlpt}, we know that interpreting an rMLP (for the problem MCR) with $9t + 27 = 3(3t+8) + 3$ is $\W(\Maj)[3t+8]$-hard. On the other hand, by using the same proposition, the problem of interpreting an rMLP with $3t+3$ layers is contained in $\W(\Maj)[3t+7]$. But by hypothesis, $\W(\Maj)[3t+7] \subsetneq \W(\Maj)[3t+8]$, which is enough to conclude the proof.

\end{appendices}


\end{document}